\documentclass[twoside]{article}

\usepackage{amsmath,amsfonts,bm}

\def\eqref#1{equation~\ref{#1}}

\def\1{\bm{1}}

\DeclareMathAlphabet{\mathsfit}{\encodingdefault}{\sfdefault}{m}{sl}
\SetMathAlphabet{\mathsfit}{bold}{\encodingdefault}{\sfdefault}{bx}{n}

\usepackage{rotating}
\usepackage{hyperref}
\usepackage{wrapfig}
\usepackage{url}
\usepackage{graphbox}
\usepackage{natbib}

\allowdisplaybreaks

\usepackage{amsmath}
\usepackage{amssymb}
\usepackage{mathtools}
\usepackage{amsthm}

\usepackage[capitalize,noabbrev]{cleveref}

\usepackage{multirow}

\theoremstyle{plain}
\newtheorem{theorem}{Theorem}[section]
\newtheorem{proposition}[theorem]{Proposition}
\newtheorem{lemma}[theorem]{Lemma}

\theoremstyle{definition}
\newtheorem{definition}[theorem]{Definition}
\newtheorem{assumption}[theorem]{Assumption}

\theoremstyle{remark}

\usepackage[textsize=tiny]{todonotes}

\def\sgn{\mathop{\rm sgn}}

\usepackage[accepted]{aistats2024}
\begin{document}

\twocolumn[
\aistatstitle{Better Representations via Adversarial Training in Pre-Training: A Theoretical Perspective}
\aistatsauthor{Yue Xing \And Xiaofeng Lin \And Qifan Song}
\aistatsaddress{Michigan State University \And University of California, Los Angeles \And Purdue University}

\aistatsauthor{Yi Xu \And Belinda Zeng \And Guang Cheng}

\aistatsaddress{Amazon Search-M5 \And Amazon Search-M5 \And University of California, Los Angeles}\vspace{-0.3in}
\aistatsaddress{\And \And Amazon Search-M5}
]


\author{ Yue Xing \\ \texttt{xingyue1@msu.edu} \\
       Department of Statistics and Probability\\
      Michigan State Univerity
      \and
       Xiaofeng Lin\\ \texttt{bernardo1998@g.ucla.edu} \\
       Department of Statistics\\
      University of California, Los Angeles
      \and
      Qifan Song\\  \texttt{qfsong@purdue.edu} \\
       Department of Statistics\\
      Purdue University
      \and
       Yi Xu \\  \texttt{yxaamzn@amazon.com} \\
      Amazon Search-M5
      \and
       Belinda Zeng \\ \texttt{zengb@amazon.com} \\
      Amazon Search-M5
      \and
       Guang Cheng \\\texttt{guangcheng@ucla.edu} \\
       Department of Statistics\\
      University of California, Los Angeles\\
      Amazon Search-M5
      }

\begin{abstract}
     Pre-training is known to generate universal representations for downstream tasks in large-scale deep learning such as large language models. Existing literature, e.g., \cite{kim2020adversarial}, empirically observe that the downstream tasks can inherit the adversarial robustness of the pre-trained model. We provide theoretical justifications for this robustness inheritance phenomenon. Our theoretical results reveal that feature purification plays an important role in connecting the adversarial robustness of the pre-trained model and the downstream tasks in two-layer neural networks. Specifically, we show that (i) with adversarial training, each hidden node tends to pick only one (or a few) feature; (ii) without adversarial training, the hidden nodes can be vulnerable to attacks. This observation is valid for both supervised pre-training and contrastive learning. With purified nodes, it turns out that clean training is enough to achieve adversarial robustness in downstream tasks.

\end{abstract}
\section{Introduction}

Adversarial training is a popular way to improve the adversarial robustness of modern machine learning models. However, compared to clean training, the computation cost of adversarial training is much higher. For example, using a single GPU to train ResNet18 for CIFAR-10, clean training takes 1 hour, but adversarial training uses 20 hours for 200 epochs \citep{rice2020overfitting}. 

One possible way to train an adversarially robust neural network with a lower cost is to utilize pre-trained models. That is, we use a large amount of (possibly un-labelled) pre-training data to first train a ``general-purpose'' neural network model; then, for any specific downstream task, we only need to adapt the last one or two layers according to the (often labelled) downstream data. The computation burden is then moved from downstream users to the pre-training phase. Such a strategy has been widely adopted in the training of large language models such as the GPT series. Please see more discussions in the recent review on {foundation model} \citep{bommasani2021opportunities}.

If the statistical properties of pre-trained models can be inherited, then the pre-training strategy can also greatly simplify the training of robust downstream models, as long as the pre-trained models possess proper adversarial robustness.
Recent literature, e.g., \cite{zhao2022blessing} shows that clean pre-training can improve the sample efficiency of the downstream tasks, while for adversarial training, it is empirically observed such an inheritance of robustness from pre-training models to downstream task training \citep{shafahi2019adversarially,chen2021cartl,salman2020adversarially,deng2021adversarial,zhang2021pre,kim2020adversarial,fan2021does}. Unlike the existing works, this paper aims to provide theoretical validation for this robustness inheritance phenomenon.

{While most theoretical studies of adversarial training are from statistical/optimization perspectives, \citet{allen2020feature} studies how adversarial training improves supervised learning in neural networks via {\em feature purification}. 
The observed data can be viewed as a mixture of semantic features, and the response is directly related to the features rather than the observed data.  It is justified that 
in clean training, each node learns a mixture of features, i.e., \textit{no feature purification}. Rather, the nodes will be purified in the adversarial training in the sense that each only learns one or a few features, i.e., \textit{feature purification happens}.}

{
Different from \cite{allen2020feature} which studies supervised learning, this work aims to know whether the benefit of feature purification appear in self-supervised pre-training methods, e.g., contrastive learning. In addition, after obtaining a pre-trained robust model, we wonder how the downstream task inherits the robustness using adversarial pre-training.
}

 Our contributions are as follows:

~(1) { We provide a theoretical framework  to verify that, the design of adversarial loss promotes feature purification. Beyond the work of \cite{allen2020feature} which studies the evolution of the training trajectory given a specific optimizer,  we directly consider the optimal solution to focus on the best possible performance of adversarial training regardless of the optimization algorithm.}  For the class of neural networks we consider, there are many possible optimal models to minimize the clean population risk, but only those with minimal adversarial loss achieve feature purification. (Section \ref{sec:intuition})

{
~(2) We extend our analysis to contrastive learning and verify that many non-robust models achieve the best clean performance while the robust ones have purified hidden nodes. An interesting observation is that adversarial training purifies the neural network via negative (dissimilar) pairs of data, and the loss of positive (similar) pairs of data is almost resistant to adversarial attack. This is a different observation compared to \cite{allen2020feature}. (Section \ref{sec:contra})
}

~(3) Our results also show that when the pre-trained model perfectly purifies the hidden nodes, we can achieve good model robustness when the downstream tasks are trained using clean training. (Section \ref{sec:downstream})

\section{Related Works}

\paragraph{Feature Purification and Better Representation}
Some related literature touches on similar questions as our targets but with a different purpose from ours. {\citet{wen2021toward} shows that contrastive learning can purify features using RandomMask. A detailed discussion on the advantages/disadvantages of adversarial training and RandomMask can be found in Section \ref{rem:li}. Another related work is \cite{deng2021adversarial}, which shows that adversarial training helps select better features from {individual tasks}. This is different from ours as it does not work on nonlinear neural networks. }
\vspace{-0.1in}

\paragraph{Adversarial training}
There are fruitful studies in the area of adversarial training. For methodology,  there are many techniques, e.g., \cite{goodfellow2014explaining,zhang2019theoretically,wang2019improving,cai2018curriculum,zhang2020attacks,carmon2019unlabeled,gowal2021improving,mo2022adversarial,wang2022improving}.
Theoretical investigations have also been conducted for adversarial training from different perspectives. For instance, \citet{chen2020more, javanmard2020precise,taheri2021statistical,yin2018rademacher,raghunathan2019adversarial,najafi2019robustness,min2020curious,hendrycks2019using,dan2020sharp,wu2020revisiting,javanmard2021adversarial,deng2021improving,javanmard2022precise} study the statistical properties of adversarial training, \cite{sinha2018certifying,wang2019convergence,xing2021generalization,xing2021algorithmic,xiao2022stability,xiao2022adaptive} study the optimization aspect of adversarial training, \cite{zhang2020over,wu2020does,xiao2021adversarial} work on theoretical issues related to adversarial training with deep learning. 
\vspace{-0.1in}

\paragraph{Contrastive learning}

Contrastive learning is a popular self-supervised learning algorithm. It uses unlabeled images to train representations that distinguish different images invariant to non-semantic transformations (\citealp{mikolov2013distributed,oord2018representation,arora2019theoretical,dai2017contrastive,chen2020simple,tian2020makes,chen2020simple,khosla2020supervised,haochen2021provable,chuang2020debiased,xiao2020should,li2020prototypical}). Beside empirical studies, there are also many theoretical studies, e.g., \cite{saunshi2019theoretical,haochen2021provable,haochen2022beyond,shen2022connect,haochen2022theoretical,saunshi2022understanding}. Other related studies in adversarial training with contrastive learning can also be found in \cite{alayrac2019labels,ho2020contrastive,jiang2020robust,cemgil2019adversarially,petrov2022robustness,nguyen2022task}.

\section{Model Setups}
This section defines data generation model, neural network, and adversarial training for supervised learning. 

\vspace{-0.05in}
\subsection{Data Generation Model}\label{sec:data}
\vspace{-0.05in}

\begin{figure}
    \centering
    \includegraphics[width=\linewidth]{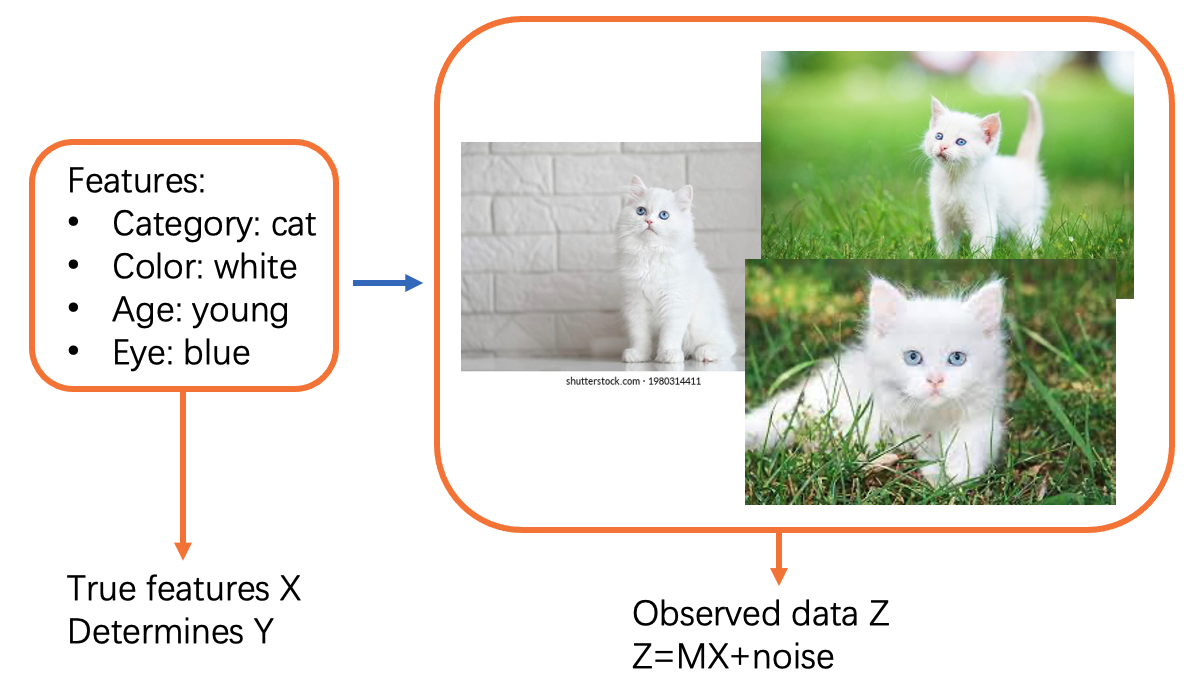}
    \vspace{-0.3in}
    \caption{ A proof-of-concept example of the Sparse Coding Model. For the categorical features, one can reshape it to a sparse feature vector. }
    \vspace{-0.15in}
    \label{fig:sparse_coding}
\end{figure}
 
We consider the following data generation model. There exists some underlying \textbf{true features} $X\in\mathbb{R}^d$, such that $Z=MX+\xi$ for a unitary matrix $M$, and the response $Y$ is directly determined by $X$. However, instead of observing $X$, we observe transformed noisy features $Z\in\mathbb{R}^d$ and the response $Y\in\mathbb{R}$. An illustration can be found in Figure \ref{fig:sparse_coding}. 

The relationship between $X$ and $Y$ is as follows:

~(1) Regression: $Y=\theta_0^{\top}X+\varepsilon$ for Gaussian noise $\varepsilon$.

~(2) Classification: $Y\sim Bern(1/(1+\exp(\theta_0^{\top}X)))$.

We impose the following assumption on the data.

\begin{assumption}[Sparse coding model]\label{assumption:x}
The model of $(X,Z,Y)$ satisfies the following conditions:

~(1) The coordinates of $X$ are i.i.d. symmetric variables, and $|X_i|\in\{0\}\cup [1/\sqrt{k},1]$ for some sparsity parameter $k$. Moreover, $P(|X_i|\neq 0)=\Theta(k/d)$,
$ \mathbb{E}X_i^2=\Theta(1/d)$, $\mathbb{E}|X_i|=\Theta(1/\sqrt{k})$, and $\mathbb{E}|X_i|^3=\Theta(1/(d\sqrt{k}))$.

~(2) The noise $\xi$ follows i.i.d. $N(0,\zeta^2 I_d/d)$ for $\zeta>0$.

~(3) Each coordinate of $\theta_0$ satisfies $(\theta_0)_i=\Theta(1)$.
\end{assumption}

In Assumption \ref{assumption:x} (1), we assume $X$ is a sparse signal. In general, there are $O(k)$ \textbf{active} features (i.e., non-zero $X_i$) in a realization of $X$. In the later results, we always assume $k\ll d$. In addition, together with (2), we have $\|X\|=O_p(1)$, $\|Z\|=O_p(1)$, and $\|\xi\|=O_p(1)$, i.e., the total magnitudes of the features, observed data, and noise are comparable. Assumption \ref{assumption:x} (3) indicates that all features are important in determining $Y$.

Assumption \ref{assumption:x} is similar to the model considered in \cite{allen2020feature}. We impose a constraint on the third moment of $X_i$ to use concentration bounds in Lemma \ref{lem:lem}, and assume  a symmetric distribution to ease the contrastive learning (Lemma \ref{lem:contra_basic}). { Similar sparse coding models have a long history in literature, e.g., \cite{hyvarinen1998image}.}

\vspace{-0.05in}
\subsection{Two-Layer Neural Network}
\vspace{-0.05in}

We use a two-layer neural network to fit the model. In particular, given an input $z$, \vspace{-0.1in}$$f_{W,b}(z) = \sum_{h=1}^H a_h\sigma(z^{\top}W_h,b_h),$$
where $a_h=1$ for all $h$. In the pre-training stage, we use lazy training and do not update $a_h$. The vector $b=(b_1,b_2,\ldots,b_H)$ is the intercept\footnote{We use the term ``intercept" to describe $b$ to avoid confusion with ``bias" in statistics.} term in each node, and $W=(W_1\mid \ldots \mid W_H)$ is the coefficient matrix. In later sections, besides using $W_h$ as the $h$th column of $W$, we also define $W_{i,:}$ as the $i$th row and use $W_{:,h}$ as the $h$th column of $W$ to avoid confusion when needed. Similar notations will be used for other matrices.

To simplify the derivation, we mainly consider the following activation function, for $v,e\in\mathbb{R}$,
\begin{eqnarray}\label{eqn:activation}
\sigma(v,e)=v 1\{|v|\geq e\}.
\end{eqnarray}
Compared to an identity mapping, (\ref{eqn:activation}) has  an extra ``gate parameter" $e$ to screen out weak signals. When $|v|>e$, the hidden node is \textbf{activated}.

\vspace{-0.05in}
\subsection{Adversarial Training}
\vspace{-0.05in}

We consider $\mathcal{L}_2$ fast gradient attack (FGM) with attack strength $\epsilon$, i.e., given the current model $f$ and loss function $l$, for each sample $(z,y)$, the attack is 
$$ \delta_2 = \epsilon (\partial l/\partial z)/{\|\partial l/\partial z\|},$$
where $\|\cdot\|$ is the $\mathcal{L}_2$ norm. In the models we consider, when approaching the optimal solution to minimize clean/adversarial loss, the FGM is the best attack. 

Besides the adversarial attack, we also define the corresponding adversarial loss as
$l_\epsilon(z,y;f) = l( z+\delta, y;f ).$
Besides $\mathcal{L}_2$ attack, some discussions can also be found in the appendix Section \ref{sec:appendix:linf} if $\delta$ is the $\mathcal{L}_{\infty}$ attack (i.e., Fast Gradient Signed Method (FGSM)).
When $\epsilon=0$, the loss $l_0$ is reduced to $l$, and represents the clean loss. Details of contrastive learning are in Section \ref{sec:contra}.

For clean and adversarial training in this paper, we use ``clean training" to minimize the clean loss and  ``adversarial training" to minimize the adversarial loss.  

\section{Feature Purification}\label{sec:intuition}
This section aims to provide basic intuitions 
on
(i) why the activation function (\ref{eqn:activation}) and ReLU are preferred over linear networks, 
and (ii) why adversarial training can purify features. {While the high-level ideas are similar to \cite{allen2020feature}, we restate them via different technical tools so that it can be carried over to the later sections of contrastive learning and downstream study.}

\vspace{-0.05in}
\subsection{Screen Out Noise}\label{sec:intuition_nonline}
\vspace{-0.05in}

The basic rationale of why the activation function in (\ref{eqn:activation}) (or ReLU) is that it can screen out the noise $\xi$. Intuitively, the noise $\xi$ in $Z$ only contributes to a negligible noise in hidden nodes, which can be screened out by  a proper ``gate parameter'' $b_h$. 

To explain more details, we introduce the notations first. Based on the data generation model, assume $z$ is a realization of $Z$, then we can define $U=M^{\top} W \in\mathbb{R}^{d\times H}$ and rewrite $f_{W,b}(z)$ as
\begin{eqnarray*}
f_{W,b}(z) = \sigma(z^{\top}W,b)a = \sigma( x^{\top} U + \xi W,b )a.
\end{eqnarray*}
To interpret $U$, for each hidden node $h$, the column $U_h$ represents the strength of each feature in the hidden node. Note that the noise $\xi^{\top} W_h\sim N(0,\zeta^2\|U_h\|^2/d)$. When $b_h\gg \|U_h\|/\sqrt{d}$, the noise alone is not able to activate the hidden node. On the other hand, for an active feature $X_i\neq0$, we have $|U_{i,h}X_i|\geq |U_{i,j}|/\sqrt{k}$, which can be much larger than $\zeta\|U_h\|/\sqrt{d}$ for proper $U_h$. As a result, under a reasonably tuned $b_h$, the active features will survive the screening effect and activate the hidden node. Noise may pass through the screening of a node and corrupt the prediction only when this node also contains other active features, but the contribution of noise ($W_h^{\top}\xi$) will be negligible compared with other active features. 

 To simplify our analysis, we impose the following assumption to focus on strong features:

\begin{definition}\label{assumption:nn}
Define $\mathcal{M}$ as the set of two-layer neural networks such that, for any node $h$, 

~(1) The intercept is within a proper range, i.e. $b_h\ll\|U_h\|/\sqrt{k\|U_h\|_0}$, and

~(2) $b_h\gg \|U_h\|/\sqrt{k\|U_h\|_0}/\log d\gg \|U_h\|/\sqrt{d}$.

~(3) There are at most $m^*$ of features of $X$ learned by each hidden node, i.e., $\|U_h\|_0\leq m^*$ for all $h=1,\ldots,H$, and $m^*=o(d/k)$. All hidden nodes are non-zero and $H\gg d$.

~(4)  Non-zero $U_{i,h}$'s have the same sign for the same $i$ and $|U_{i,h}|=\Theta(\gamma)$.
\end{definition}

The conditions (1) and (2) in Definition \ref{assumption:nn} match the intuition above to conduct screening. The conditions (3) and (4) are for the simplicity of the derivation. 

For the ReLU activation function, it is similar to the activation we considered in \ref{eqn:activation}, and the related discussion is postponed to Section \ref{sec:dics}.

\vspace{-0.05in}
\subsection{Purified Nodes Lead to Robustness}
\vspace{-0.05in}

{
To intuitively understand why feature purification improves adversarial robustness, a graphical illustration can be found in Figure \ref{fig:purification}. With either purified/unpurified hidden nodes, the active features $X_1$ and $X_3$ will always be attacked. With purified features, adding an attack does not impact the inactive features. With unpurified features, the inactive features can also be attacked.
}

\begin{figure}
    \centering
    \includegraphics[width=\linewidth]{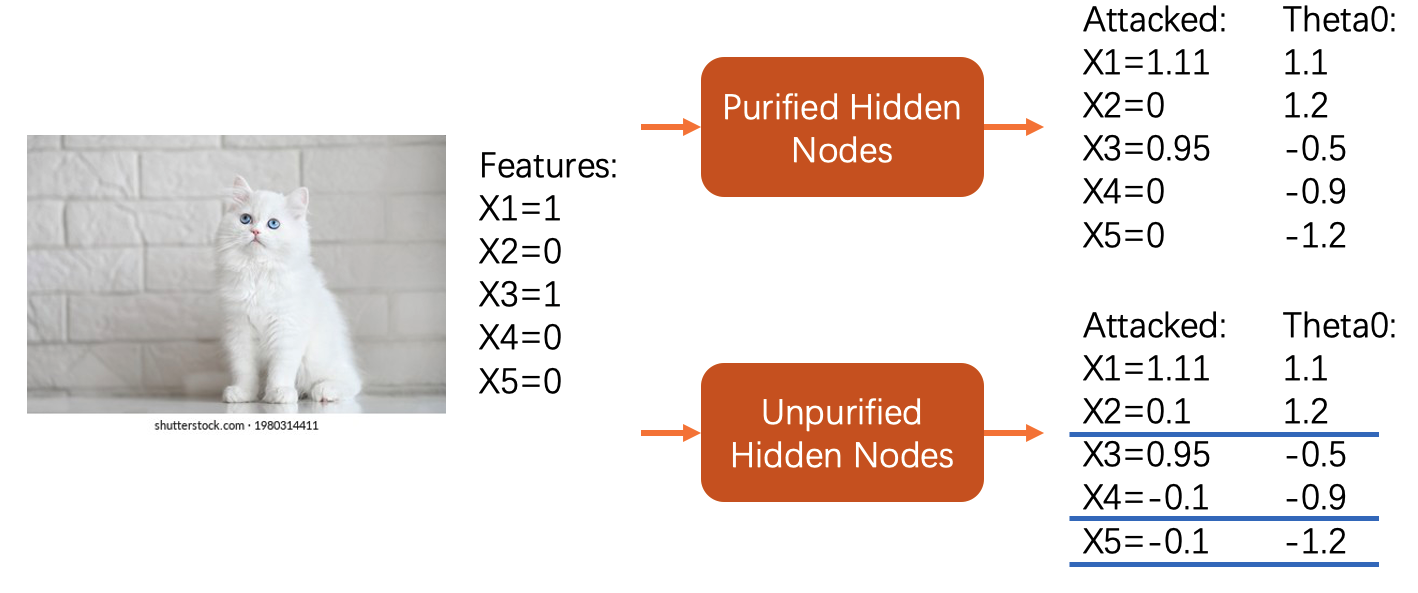}
    \vspace{-0.3in}
    \caption{With purified hidden nodes, only the active features will be attacked, and the resulting adversarial loss is small. With unpurified hidden nodes, inactive features will also be impacted. Note that we transform the attack on the observable $Z$ back to its features $X$, to compare with $\theta_0$.}
    \label{fig:purification}
    \vspace{-0.1in}
\end{figure}

The following is extended from the above intuition:
\begin{lemma}\label{thm:adv_ideal}
Assume $ \epsilon = O(1/(\log(d)\sqrt{m^* k}))$, and $(W,b)\in\mathcal{M}$. Denote $\mathcal{X}$ as the set of coordinate $i$ where $|X_i|>0$. 
Assume $Ua=\theta$,  $\|\theta\|_{\infty}=\Theta(1)$.
 With probability tending to 1 over the randomness of $\xi$ and $X$, 
\begin{eqnarray}
\Delta_{W,b}(z,y)=\epsilon\frac{\partial l}{\partial f_{W,b}}\left\|a^{\top}\text{diag}(\mathbb{I}(W^{\top}z ,b)) W^{\top}\right\|_2+o,\label{eqn:lem:1}
\end{eqnarray}
where ``$o$'' represents a negligible term caused by the curvature of the loss.  In probability,
\begin{eqnarray}
\|\theta_{\mathcal{X}}\|_2\leq \|a^{\top}\text{diag}(\mathbb{I}(U^{\top}X,b))U^{\top}\|_2\leq\|\theta\|_2\label{eqn:lem:2},
\end{eqnarray}
the $\theta_\mathcal{X}$ is the vector of the coordinates of the $\theta$ in $\mathcal{X}$.

The \textbf{left} equation holds (i.e., highest robustness) only when the matrix $U$ is sparse, i.e., $\|U_h\|_0\leq1$ for every hidden node $h$. When all hidden nodes are activated, the \textbf{right} equation 
holds.
\end{lemma}

Lemma \ref{thm:adv_ideal} illustrates how the effectiveness of attack (i.e.,.$\Delta_{W,b}$) is affected by the neural network. If the hidden nodes are purified, then the neural network is more robust, and the increase from $l_0$ to $l_\epsilon$ is small. If not, more hidden nodes are activated, leaking more weight information.

There are two key claims to prove Lemma \ref{thm:adv_ideal}: (i) to show equation (\ref{eqn:lem:1}), we show that in probability, every activated hidden node will not be deactivated by the attack (Lemma \ref{lem:3}), and (ii) to show equation (\ref{eqn:lem:2}), we show that in probability, every hidden node is activated as long as some of its learned  features are non-zero (Lemma \ref{lem:2}). The proof for Lemma \ref{thm:adv_ideal} and all the following theorems and propositions can be found in Appendix \ref{sec:appendix:proof_l2}.

\vspace{-0.05in}
\subsection{Purification in Supervised Learning}\label{sec:pretrain_super}
\vspace{-0.05in}

We use square loss and absolute loss for regression and logistic loss for classification. Given the loss function as $l$, the task is to minimize $\mathbb{E} l_{\epsilon}( Z,Y; W,b )$. 

{ Thanks to Lemma \ref{thm:adv_ideal}, we are able to study the clean and robust performance of neural networks. Since our main focus is on contrastive learning, we provide an informal statement for supervised learning below, and postpone the formal theorems to Appendix \ref{sec:sup}.}
\begin{theorem}[Informal Statement]\label{thm:informal}
    For some $(W,b)\in\mathcal{M}$ satisfying $Ua=\theta_0$, for square loss, absolute loss, and logistic regression, we define a vanishing term $\psi$ as
\begin{eqnarray*}
    \mathbb{E}l_0(Z,Y;W,b)=\mathbb{E}l_0(X,Y;\theta_0)+O(\psi).
\end{eqnarray*}
There exists many $(W,b)\in\mathcal{M}$ such that the clean loss is $O(\psi)$-close to its minimum, while the adversarial loss is $\Theta(\epsilon\sqrt{m^*k})$-close to its minimum. When using adversarial training so that the adversarial loss is $O(\psi)$-close to its minimum, the clean loss is $\Theta(\epsilon\sqrt{k})$-close to its minimum, and at most $o(1)$ proportion of hidden nodes learn more than 1 feature.
\end{theorem}

\section{Purification in Contrastive Learning}\label{sec:contra}
In this section, we show that in contrastive learning, clean training does not intend to purify the neural networks, and adversarial training does.

\vspace{-0.05in}
\subsection{Model Setup}
\vspace{-0.05in}

The contrastive learning aims to learn a $g: \mathbb R^d\otimes \mathbb R^d\rightarrow \mathbb R$ to minimize the following loss
\begin{eqnarray}\label{eqn:contra_loss}
&&\mathbb{E}_Z \mathbb{E}_{Y} l(Z,Z'(Y),Y ;g)
\\&:=&\mathbb{E}_Z \mathbb{E}_{Y} \log\left(1+\exp[-Y g(Z,Z'(Y))]\right)\nonumber
\end{eqnarray}
where $Z'(Y):=Z'=MX'+\xi'$ for a noise $\xi'$ that is i.i.d. to $\xi$, and $Y$ determines whether the pair $(Z,Z')$ is similar or not, i.e., if $Y=1$, $X'=X$, otherwise, $X'$ is an independent copy of $X$. In other words, when $Y=1$, $Z$ and $Z'$ share the same true features, can be interpreted as two views of the same object $X$; when $Y=-1$,  $Z$ and $Z'$ are independent and correspond to different true features $X$ and $X'$. Note that the label $Y=\pm 1$ in contrastive learning is not the class label in the original data set. It is an artificial label, manually generated following marginal distribution  $P(Y=1)=P(Y=-1)=0.5$. The label $Y$ represents whether the two views correspond to the same sample or not.
Given a neural network parameterized by $W$, $b$ and $A$ that outputs multiple responses, the loss function $l$ is in the format of
$$l(z,z',y;W,b)= \log\left(1+\exp[-y g_{W,b}(z,z')]\right),$$
where 
\vspace{-0.1in}\begin{eqnarray}\label{eqn:gwb}
&&g_{W,b}(z,z')\\
&=&
\left(\sum_{h=1}^H A_{h,:}\sigma(W_h^{\top} z,b_h)\right)^{\top}\left(\sum_{h=1}^H A_{h,:}\sigma(W_h^{\top} z',b_h)\right)\nonumber
\end{eqnarray}
with the output layer $A\in\mathbb{R}^{H\times d}$ with the same output dimension as the data dimension. Note that parameter $A$ is not a trainable parameter since we will consider a lazy training scenario. The details will be discussed later.  Unlike the supervised task where the neural network outputs a single value, in contrastive learning, the neural network outputs a vector. 

For adversarial attack, we again consider the FGM attack, i.e., $\delta_2=\epsilon(\partial l/\partial z)/\|\partial l/\partial z\|_2$, and the corresponding adversarial loss can be written as \vspace{-0.05in}$$l_\epsilon(z,z',y;W,b)=l(z+\delta_2,z',y;W,b).$$






\vspace{-0.05in}
\subsection{Optimal Solution and Lazy Training}
\vspace{-0.05in}

The optimal solutions for supervised learning loss and contrastive learning loss are different. 
But for contrastive learning, by \cite{tosh2021contrastive}, the optimal solution of contrastive learning (\ref{eqn:contra_loss}) is 
\vspace{-0.1in}$$g^*(z,z')=\log\left(\frac{f_{Z,Z'}(z,z')}{f_{Z}(z)f_{Z'}(z')}\right),$$
where $f_{Z}$, $f_{Z'}$, and $f_{Z,Z'}$ are the marginal and joint density functions and are not linear functions. Thus, under our Definition \ref{assumption:nn}, which imposes restrictions on $W$ and $b$ such that $\sigma(W_h^{\top} z,b_h)$ has a linear function behavior, we cannot achieve good contrastive loss with the two-layer network modeling of $g_{W,b}$. 

Based on the following lemma, the best solution of contrastive loss, among linear networks $Tx$, still enjoys a nice and tractable form under simple settings.

\begin{lemma}[Basic Properties of Contrastive Learning]\label{lem:contra_basic}
Consider the class of functions $g_T(x,x')=x^{\top}T^{\top}Tx'$ using the ground-truth feature $X$ for some matrix $T^{\top}T=PDP^{\top}$ with an orthonormal matrix $P$ and a diagonal matrix $D$. Assuming $tr(D)$ is fixed, then the best model to minimize contrastive loss $\mathbb{E}_X \mathbb{E}_{Y} \log\left(1+\exp[-Y g_T(X,X')]\right)$ satisfies $D\propto I_d$. 
\end{lemma}


Lemma \ref{lem:contra_basic} motivates us to utilize a new lazy-training method in contrastive learning to simplify the analysis. Unlike supervised pre-training, where we fix weight $a\equiv 1$, in contrastive learning, we would ensure that $M^{\top}WAA^{\top}W^{\top}M\propto I_d$ is fixed, i.e., $WAA^{\top}W^{\top}\propto I_d$. As a result, instead of completely fixing last layer parameters, in contrastive learning, we take $A=\tau W^+$ for a fixed $\tau$ while updating weight matrix $W$, where $W^+$ is the pseudo inverse of $W$.

\vspace{-0.05in}
\subsection{Similar vs Dissimlar Pairs}
\vspace{-0.05in}

For the above setting, adversarial training will help feature purification. { However, different from supervised adversarial training where the attack of all samples contributes to feature purification, in contrastive learning, only the attack on dissimilar data pairs are affected by feature purification.}
\vspace{-0.1in}

\begin{figure}
    \centering
    \includegraphics[width=\linewidth]{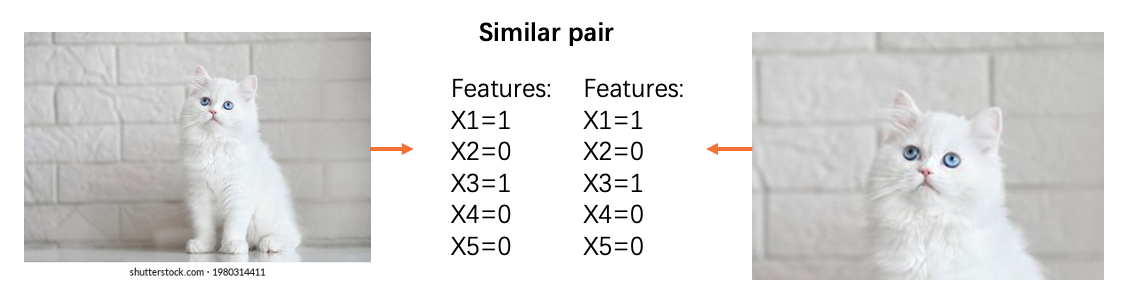}
    \includegraphics[width=\linewidth]{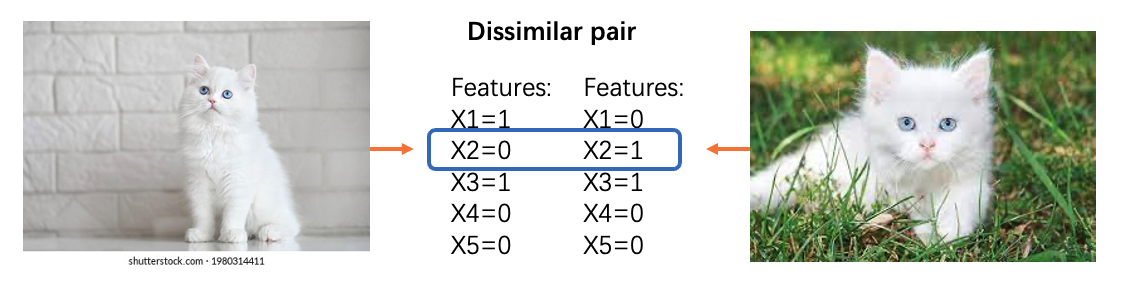}
    \vspace{-0.3in}
    \caption{Adversary attacks on dissimilar pairs, but have little effect on similar pairs.}
    \label{fig:pair}
    \vspace{-0.1in}
\end{figure}

\paragraph{Intuition}{To explain why adversarial training affects more on the loss of dissimilar pairs, we use Figure \ref{fig:pair} as an example. In Figure \ref{fig:pair}, we attack the left data and keep the right data unchanged.  Suppose that the attack changes the inactive features from 0 to another value of the left image, Given a similar pair (i.e., two views of the same data via different data augmentation and their underlying features are the same), the change of the left features is canceled when multiplying the zero feature of the right data. However, for dissimilar pairs, there is a mismatch between the features. For example, the attack changes $X_2$ of the left data to another value, when multiplying with $X_2=1$ on the right data, the product gets changed, leading to a big change in the loss value.}

\vspace{-0.1in}
\paragraph{Simulation}We also conduct a toy simulation and plot the result in the left panel of Figure \ref{fig:contra}. 

To generate $X$, we consider the following distribution. First, each coordinate of $X$ is independent of each other, and has $k/d$ probability to be non-zero. Second, given $X_i$ is non-zero, it has 1/2 probability to be positive, and we take the distribution as $\min(1,|\varepsilon|/\sqrt{k}+1/\sqrt{k})$ with $\varepsilon\sim N(0,1)$. The distribution is symmetric to $X_i<0$.

To generate $Z$, we randomly generate a unitary matrix $M$, and take $Z=MX+\xi$, with $\xi\sim N(0,\zeta^2 I_d)$. To generate $M$, we use library \texttt{pracma} in \texttt{R}. We take $(d,k,\zeta)=(1000,10,0.005)$, and generate 1000 samples in each simulation and repeat 30 times to obtain an average. To generate $Y$ for supervised learning, we take $\theta_0=\textbf{1}$, and $Y=X^{\top}\theta+ N(0,\sigma^2I_d)$ with $\sigma=0.1$. And in terms of the neural network, we take $H=10000$ hidden nodes.

 We control the average $\|U_h\|_0$ and evaluate the clean and adversarial loss. 
We plot four curves, representing the change of clean and adversarial contrastive losses for similar data pairs (i.e., $Z$ and $Z'(1)$) and dissimilar data pairs (i.e., $Z$ and $Z'(-1)$), as the number of features in each hidden node increases.
As the number of features in each hidden node gets larger, the adversarial loss for dissimilar pairs gets larger. The detailed setup, numbers in the figure, and standard errors can be found in Appendix \ref{sec:appendix:simulation}.

\begin{figure*}
    \centering
    \includegraphics[scale=0.47]{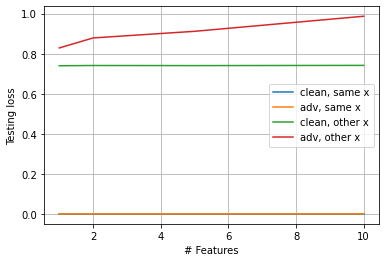}\includegraphics[scale=0.47]{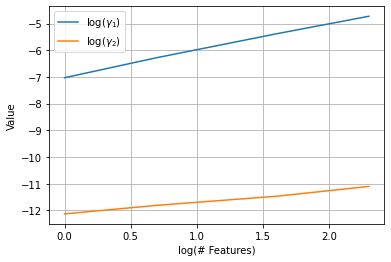}
\vspace{-0.in}
    \caption{Left: Clean/adversarial contrastive testing loss under different levels of purification of the hidden nodes, for similar data pairs (i.e., $Y=1$) and dissimilar data pairs (i.e., $Y=-1$). Note that the blue and yellow curves overlap. Right: How $\alpha$ is related to $m$. The values of $\gamma_1$ and $\gamma_2$ are assumed to be in $\Theta(\alpha)$ and  $\Theta(\alpha^2)$ respectively in Theorem \ref{lem:robust_similar}.}
    \label{fig:contra}
    \vspace{-0.2in}
\end{figure*}

\vspace{-0.1in}
\paragraph{Theory}Based on the above simulation observation and intuition, the following theorem demonstrates how an adversarial attack impacts contrastive learning.

\begin{theorem}\label{lem:robust_similar} 
Assume that $(W,b)$ satisfies Definition \ref{assumption:nn} and $|\mathcal{X}|=\Theta(k)$ where $\mathcal{X}$ denotes the set $\{i:|X_i|>0\}$.
Let $A=\tau W^+$ for some fixed $\tau>0$. If
$$U_{i,:}\text{diag}(\mathbb{I}(U_{\mathcal{X},:},\textbf{0}))U^{\top}(UU^\top)^{-1}_{:,j}=\Theta(\alpha)$$
for $i\in\mathcal{X}^c$ and $j\in\mathcal{X}$, and
$$U_{i,:}\text{diag}(\mathbb{I}(U_{\mathcal{X},:},\textbf{0}))U^{\top}(UU^\top)^{-1}_{:,j}=\Theta(\alpha^2)$$ for $i\neq j$ and $i,j\in\mathcal{X}^c$m and furthermore $\alpha=o(1/\sqrt{d})$,
then when $\epsilon=\Theta( 1/(\log(d)\sqrt{m^*k} ))$,
\begin{eqnarray}
    &&\mathbb{E}l_\epsilon(Z,Z'(1),1; W,b)\label{eqn:diff:similar}\\
   &=&\min_{(W',b')\in\mathcal{M}} \mathbb{E}l_0(Z,Z'(1),1; W',b')+\Theta(\epsilon)+O(\psi),\label{eqn:similar}\nonumber
   \end{eqnarray}
   and
\vspace{-0.1in}   \begin{eqnarray}
&&\mathbb{E}l_\epsilon(Z,Z'(-1),-1; W,b)\label{eqn:diff}\\
&=&\min_{(W',b')\in\mathcal{M}}\mathbb{E}l_0(Z,Z'(-1),-1; W',b')+O(\psi)\nonumber\\
&&+\Theta(\epsilon k^{3/2}/d)+\Theta(\epsilon\alpha^2\sqrt{d} ).\nonumber
\end{eqnarray}
\end{theorem}
\vspace{-0.1in}To make a connection between $\alpha$ and the level of purification of $U$,
we perform a simulation in Figure \ref{fig:contra} and calculate the average $U_{i,:}\text{diag}(\mathbb{I}(U_{\mathcal{X}},\textbf{0}))U^{\top}(UU^\top)^{-1}_{:,j}$ for ($i\in\mathcal{X}^c$,$j\in\mathcal{X}$) and ($i\neq j,\;i,j\in\mathcal{X}^c$) respectively, and denote $\gamma_1$ and $\gamma_2$ as the corresponding average value. From the right panel of Figure \ref{fig:contra}, one can see that  $\log(\gamma_1)$ and $\log(\gamma_2)$ are approximately linearly increasing functions of $\log(m)$. With a larger $m$, $\alpha$ will be larger. In addition, one can also see that $\log(\gamma_2)\approx \log(\gamma_1^2)$, which validates the appropriateness of our assumption in Theorem \ref{lem:robust_similar}.

To connect Theorem \ref{lem:contra_basic} and the intuition in Figure \ref{fig:contra}, when designing an attack on $Z$, $Z'(1)$ carries the information of true features $X$. As a result, the best attack on $Z$ that aims to make $Z+\delta$ dissimilar to $Z'(1)$, corresponds to the active features in $X$. For similar pairs, $\alpha$, which quantifies the associations between active and non-active features, will only have negligible effect.

On the other hand, the effect of the adversarial attack is different in (\ref{eqn:diff}) for dissimilar pairs. When $\alpha$ gets larger, $\epsilon\alpha^2\sqrt{d}$ can dominate the fixed $\epsilon k^{3/2}/d$, indicating that the neural network is more vulnerable to adversarial attack for dissimilar data pairs.

The above theorems and simulation evidences together answer our question: contrastive learning can also benefit from adversarial training.



\subsection{Discussion}
\label{rem:li}
\vspace{-0.05in}
    While we study how adversarial training purifies features in contrastive learning, another work, \cite{wen2021toward}, studies how random data augmentation improves feature purification. This augmentation is simpler to implement with a smaller computation cost, but there are two advantages of adversarial training. 
    
    First, for random data augmentation, the intercept term $b$ in the hidden node is taken to purify features only. For adversarial robustness, when $\epsilon$ gets larger, we need a larger $b$ to avoid the adversarial attack activate/deactivate hidden nodes. The intercept $b$ in adversarial training can better improve the robustness.
    
    Second, the random augmentation purifies features via decoupling the features in similar pairs, rather than in dissimilar pairs as in adversarial training. In Figure \ref{fig:contra}, the loss for similar pairs is smaller than dissimilar pairs, implying that adversarial training is more sensitive in purification.

    The data augmentation in \cite{wen2021toward} is also used in our experiments, and the clean-trained contrastive models are vulnerable to adversarial attack.



\section{Robustness in Downstream Tasks}\label{sec:downstream}
After obtaining the pre-trained model $(W,b)$, we further utilize it in a downstream supervised task.

The downstream training aims to minimize the clean loss of  downstream data $(Z_{\text{down}},Y_{\text{down}})$ w.r.t. $a$ given pre-trained weights $(W,b)$
\begin{eqnarray}
L^{W,b}(a):=\mathbb{E} L(\sigma(Z_{\text{down}}^{\top}W,b)a,Y_{\text{down}}),
\end{eqnarray}
where the loss function $L$ can be different from the one in pre-training, $Z_{\text{down}}=MX_{\text{down}}+\xi_{\text{down}}$ uses the same $M$ but possibly different $X_{\text{down}}$ satisfying the sparse coding model, $Y_{\text{down}}$ can also be different from $Y$. Denote $L_\epsilon$ as the corresponding adversarial loss, and $a^*$ as the corresponding optimal solution.


The following proposition indicates that the robustness in pre-training can be inherited.
\begin{proposition}\label{thm:downstream}
When $\epsilon=\Theta(1/(\log (d)\sqrt{m^*}k))$,

~(1) There exists $(W,b)$ that minimizes pre-training clean (supervised or contrastive) loss s.t.
        \begin{eqnarray*}
L_\epsilon^{W,b}(a^*)- L_0^{W,b}(a^*) \gg \epsilon \sqrt{k} + O(\psi).
\end{eqnarray*}
~(2) Assume $(W,b)\in\mathcal{M}$ minimizes the adversarial pre-training loss and $\sup_h \|U_h\|_0=1$, then
    \begin{eqnarray*}
L_\epsilon^{W,b}(a^*)- L_0^{W,b}(a^*)  = \Theta(\epsilon \sqrt{k})+O(\psi).
\end{eqnarray*}
\end{proposition}
\vspace{-0.1in}The proof of Proposition \ref{thm:downstream} is similar to Theorem \ref{thm:clean}.

Proposition \ref{thm:downstream} illustrates two observations. First, using clean loss in the pre-training, since one cannot purify the neural network, the corresponding downstream training is not robust. Second, if we obtain a purified neural network, the downstream model is robust.

\section{Real-Data Experiments}\label{sec:experiment}

{Our experiments aim to justify (1) the robustness inheritance phenomenon in Section \ref{sec:downstream}; (2) Adversarial training purifies the features  (Section \ref{sec:pretrain_super} and \ref{sec:contra}).}

\vspace{-0.05in}
\subsection{Experimental Setups}
\vspace{-0.05in}

We perform supervised learning \cite{rice2020overfitting}\footnote{\url{https://github.com/locuslab/robust_overfitting} } and contrastive learning\cite{kim2020adversarial}\footnote{\url{https://github.com/Kim-Minseon/RoCL} } pre-training (i.e., pre-training consists of a clean training phase, followed by an adversarial training phase) to verify that the hidden nodes are purified. After pre-training the neural network, we remove its last layer, train a new last layer using a supervised task (clean training), and test the adversarial robustness. 

Our tests are conducted on ResNet-18~\cite{he2016deep}. The attack method used for training and evaluation is PGD under $l_{\infty}$ norm and $\epsilon=8/255$. We use CIFAR-10, CIFAR-100, or Tiny-Imagenet in pre-training and CIFAR-10 for downstream training and testing. 
Details on training configurations are in Table \ref{tab:config} in the appendix, and we retain the same configuration used by the original GitHub repositories. We perform the training on a RTX-2080 GPU with 12GB RAM.

\begin{table}[ht]
\centering

  \begin{tabular}{lccrr}
  \hline Pre-train & Pre-train & Down  & Acc & Robust \\ 
  \hline
  \multirow{6}{*}{CIFAR10} & Clean & Clean  & {0.955} & 0.001 \\ 
  & Clean & Adv Sup  & 0.477 & 0.109 \\ 
   & Adv Sup & -  & 0.810 & 0.495 \\
   & Adv Sup& Clean  & {0.847} & {0.429} \\
   & Adv Sup & Adv Sup  & 0.836 & 0.484 \\
   & Adv Contra& Clean  & {0.831} & {0.393} \\
   & Adv Contra&  Adv Sup & 0.807 & 0.462 \\
  
   \hline
\end{tabular}

\caption{Robustness and accuracy in CIFAR-10 downstream task for different pre-training setups.  ``Pre-train'' and ``Downstream'' indicate the method of pre-training and the downstream task. ``Adv'' stands for adversarial training. ``Sup'' and ``Contra'' stands for supervised and constrastive learning. } 
\label{tab:adv_pretrain}

\end{table}

{Table \ref{tab:adv_pretrain} shows the training results for CIFAR-10. For the results of using CIFAR-100 or Tiny-Imagenet in the pre-training, we postpone them to Table \ref{tab:robust} in the appendix due to the page limit. In Table \ref{tab:adv_pretrain}, we evaluate the clean accuracy (Acc) and robust accuracy (Robust) in the testing dataset. For both supervised and contrastive adversarial pre-training training + clean downstream training, we observe higher robustness against PGD attacks than clean pre-training, despite minor losses in standard accuracy. This verifies the robustness inheritance phenomenon.}

We also provide benchmarks for comparison. First, clean pre-training + clean downstream training together result in near-zero robustness. Second, clean pre-training + adversarial downstream training increased robustness by 10\%, but at the cost of drastically decreased clean accuracy (-47.8\%), since the learning capacity of downstream linear layer is limited. Third, when the downstream tasks are also trained in an adversarial manner, compared with clean downstream training, the robustness increases by 5.5\% in the task following supervised adversarial pre-training and 6.9\% in the task following adversarial contrastive pre-training, which means that we are not losing too much from using clean training in the downstream tasks. Finally, the robustness is only slightly higher compared to adversarial training from scratch (49.5\%). 

Table \ref{tab:adv_pretrain} also provides the results when using CIFAR-100 in the pre-training. The observations are similar to the case of CIFAR-10. Similar results can be found in Table \ref{tab:advsup_pretrain_kernel7} in the appendix for a different input layer kernel size. {In contrast, Section \ref{exp:aug} shows that data augmentation method \citep{wen2021toward} solely cannot effectively improve robustness.}

\begin{figure}[!ht]
    \centering\vspace{-0.1in}
    \includegraphics[width=0.4\textwidth]{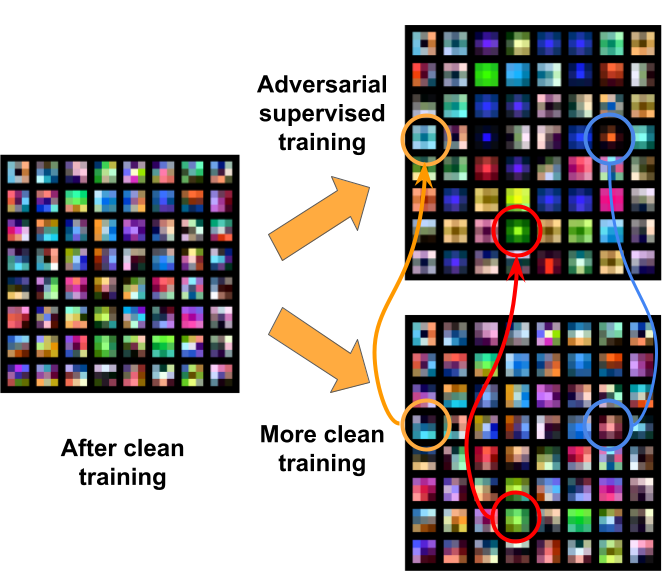}
    \caption{Learned features in the input convolutional layer trained on CIFAR-10.  } \vspace{-0.1in}
    \label{fig:conv1_pretrained}
\end{figure}

In addition to the numerical robustness result, we also visualize the trained neural networks to demonstrate the feature purification effect. Figure \ref{fig:conv1_pretrained} visualizes the features in the input convolutional layer learned from adversarial and clean pre-training. The features learned from adversarial training tend to have fewer types of colors in one cell, showing purification effects. Features in multiple filters (see the blocks marked with circles for examples) become highly concentrated, reducing the small perturbations around the center points. More figures of purification can be found in Figure \ref{fig:conv1_pretrained_cifar100}, \ref{fig:conv1_pretrained_cifar10_contrastive} and \ref{fig:conv1_pretrained_cifar100_contrastive} in Appendix \ref{sec:exp}.


\section{Conclusion}\label{sec:discussion}

In this study, we consider the feature purification effect of adversarial training in supervised/self-supervised pre-training, and the robustness inheritance in the downstream clean trained task. Both theory and experiments  demonstrate the feature purification phenomenon. As for future direction, while we consider adversarial pre-training and clean fine-tuning, it can still be burdensome for the pre-trained model provider to train a robust model. Thus, it is interesting to study the performance of clean pre-training and adversarial fine-tuning, which is needed when the pre-training is expensive, e.g., foundation models. As mentioned in Table \ref{tab:adv_pretrain}, when simply using clean pre-training with adversarial fine-tuning, the robustness cannot be effectively improved. Other methods may be considered to improve the robustness.
\bibliographystyle{asa}
\bibliography{regression}

\begin{thebibliography}{69}
\newcommand{\enquote}[1]{``#1''}
\expandafter\ifx\csname natexlab\endcsname\relax\def\natexlab#1{#1}\fi

\bibitem[{Alayrac et~al.(2019)Alayrac, Uesato, Huang, Fawzi, Stanforth, and
  Kohli}]{alayrac2019labels}
Alayrac, J.-B., Uesato, J., Huang, P.-S., Fawzi, A., Stanforth, R., and Kohli,
  P. (2019), \enquote{Are labels required for improving adversarial
  robustness?} \textit{Advances in Neural Information Processing Systems}, 32.

\bibitem[{Allen-Zhu and Li(2020)}]{allen2020feature}
Allen-Zhu, Z. and Li, Y. (2020), \enquote{Feature Purification: How Adversarial
  Training Performs Robust Deep Learning,} \textit{arXiv preprint
  arXiv:2005.10190}.

\bibitem[{Arora et~al.(2019)Arora, Khandeparkar, Khodak, Plevrakis, and
  Saunshi}]{arora2019theoretical}
Arora, S., Khandeparkar, H., Khodak, M., Plevrakis, O., and Saunshi, N. (2019),
  \enquote{A theoretical analysis of contrastive unsupervised representation
  learning,} \textit{arXiv preprint arXiv:1902.09229}.

\bibitem[{Bommasani et~al.(2021)Bommasani, Hudson, Adeli, Altman, Arora, von
  Arx, Bernstein, Bohg, Bosselut, Brunskill,
  et~al.}]{bommasani2021opportunities}
Bommasani, R., Hudson, D.~A., Adeli, E., Altman, R., Arora, S., von Arx, S.,
  Bernstein, M.~S., Bohg, J., Bosselut, A., Brunskill, E., et~al. (2021),
  \enquote{On the opportunities and risks of foundation models,} \textit{arXiv
  preprint arXiv:2108.07258}.

\bibitem[{Cai et~al.(2018)Cai, Du, Liu, and Song}]{cai2018curriculum}
Cai, Q.-Z., Du, M., Liu, C., and Song, D. (2018), \enquote{Curriculum
  adversarial training,} \textit{arXiv preprint arXiv:1805.04807}.

\bibitem[{Carmon et~al.(2019)Carmon, Raghunathan, Schmidt, Duchi, and
  Liang}]{carmon2019unlabeled}
Carmon, Y., Raghunathan, A., Schmidt, L., Duchi, J.~C., and Liang, P.~S.
  (2019), \enquote{Unlabeled data improves adversarial robustness,} in
  \textit{Advances in Neural Information Processing Systems}, pp. 11192--11203.

\bibitem[{Cemgil et~al.(2019)Cemgil, Ghaisas, Dvijotham, and
  Kohli}]{cemgil2019adversarially}
Cemgil, T., Ghaisas, S., Dvijotham, K.~D., and Kohli, P. (2019),
  \enquote{Adversarially robust representations with smooth encoders,} in
  \textit{International Conference on Learning Representations}.

\bibitem[{Chen et~al.(2021)Chen, Hu, Wang, Yinli, Wang, Shen, and
  Li}]{chen2021cartl}
Chen, D., Hu, H., Wang, Q., Yinli, L., Wang, C., Shen, C., and Li, Q. (2021),
  \enquote{CARTL: Cooperative Adversarially-Robust Transfer Learning,} in
  \textit{International Conference on Machine Learning}, PMLR, pp. 1640--1650.

\bibitem[{Chen et~al.(2020{\natexlab{a}})Chen, Min, Zhang, and
  Karbasi}]{chen2020more}
Chen, L., Min, Y., Zhang, M., and Karbasi, A. (2020{\natexlab{a}}),
  \enquote{More data can expand the generalization gap between adversarially
  robust and standard models,} in \textit{International Conference on Machine
  Learning}, PMLR, pp. 1670--1680.

\bibitem[{Chen et~al.(2020{\natexlab{b}})Chen, Kornblith, Norouzi, and
  Hinton}]{chen2020simple}
Chen, T., Kornblith, S., Norouzi, M., and Hinton, G. (2020{\natexlab{b}}),
  \enquote{A simple framework for contrastive learning of visual
  representations,} in \textit{International conference on machine learning},
  PMLR, pp. 1597--1607.

\bibitem[{Chuang et~al.(2020)Chuang, Robinson, Yen-Chen, Torralba, and
  Jegelka}]{chuang2020debiased}
Chuang, C.-Y., Robinson, J., Yen-Chen, L., Torralba, A., and Jegelka, S.
  (2020), \enquote{Debiased contrastive learning,} \textit{arXiv preprint
  arXiv:2007.00224}.

\bibitem[{Dai and Lin(2017)}]{dai2017contrastive}
Dai, B. and Lin, D. (2017), \enquote{Contrastive learning for image
  captioning,} \textit{arXiv preprint arXiv:1710.02534}.

\bibitem[{Dan et~al.(2020)Dan, Wei, and Ravikumar}]{dan2020sharp}
Dan, C., Wei, Y., and Ravikumar, P. (2020), \enquote{Sharp Statistical
  Guaratees for Adversarially Robust Gaussian Classification,} in
  \textit{International Conference on Machine Learning}, PMLR, pp. 2345--2355.

\bibitem[{Deng et~al.(2021{\natexlab{a}})Deng, Zhang, Ghorbani, and
  Zou}]{deng2021improving}
Deng, Z., Zhang, L., Ghorbani, A., and Zou, J. (2021{\natexlab{a}}),
  \enquote{Improving adversarial robustness via unlabeled out-of-domain data,}
  in \textit{International Conference on Artificial Intelligence and
  Statistics}, PMLR, pp. 2845--2853.

\bibitem[{Deng et~al.(2021{\natexlab{b}})Deng, Zhang, Vodrahalli, Kawaguchi,
  and Zou}]{deng2021adversarial}
Deng, Z., Zhang, L., Vodrahalli, K., Kawaguchi, K., and Zou, J.
  (2021{\natexlab{b}}), \enquote{Adversarial Training Helps Transfer Learning
  via Better Representations,} \textit{arXiv preprint arXiv:2106.10189}.

\bibitem[{Fan et~al.(2021)Fan, Liu, Chen, Zhang, and Gan}]{fan2021does}
Fan, L., Liu, S., Chen, P.-Y., Zhang, G., and Gan, C. (2021), \enquote{When
  Does Contrastive Learning Preserve Adversarial Robustness from Pretraining to
  Finetuning?} \textit{Advances in Neural Information Processing Systems}, 34,
  21480--21492.

\bibitem[{Goodfellow et~al.(2015)Goodfellow, Shlens, and
  Szegedy}]{goodfellow2014explaining}
Goodfellow, I.~J., Shlens, J., and Szegedy, C. (2015), \enquote{Explaining and
  Harnessing Adversarial Examples,} in \textit{3rd International Conference on
  Learning Representations}.

\bibitem[{Gowal et~al.(2021)Gowal, Rebuffi, Wiles, Stimberg, Calian, and
  Mann}]{gowal2021improving}
Gowal, S., Rebuffi, S.-A., Wiles, O., Stimberg, F., Calian, D.~A., and Mann,
  T.~A. (2021), \enquote{Improving Robustness using Generated Data,}
  \textit{Advances in Neural Information Processing Systems}, 34.

\bibitem[{Grigor’eva and Popov(2012)}]{grigor2012upper}
Grigor’eva, M. and Popov, S. (2012), \enquote{An upper bound for the absolute
  constant in the nonuniform version of the Berry-Esseen inequalities for
  nonidentically distributed summands,} in \textit{Doklady Mathematics},
  Springer, vol.~86, pp. 524--526.

\bibitem[{HaoChen and Ma(2022)}]{haochen2022theoretical}
HaoChen, J.~Z. and Ma, T. (2022), \enquote{A Theoretical Study of Inductive
  Biases in Contrastive Learning,} \textit{arXiv preprint arXiv:2211.14699}.

\bibitem[{HaoChen et~al.(2021)HaoChen, Wei, Gaidon, and
  Ma}]{haochen2021provable}
HaoChen, J.~Z., Wei, C., Gaidon, A., and Ma, T. (2021), \enquote{Provable
  guarantees for self-supervised deep learning with spectral contrastive loss,}
  \textit{Advances in Neural Information Processing Systems}, 34, 5000--5011.

\bibitem[{HaoChen et~al.(2022)HaoChen, Wei, Kumar, and Ma}]{haochen2022beyond}
HaoChen, J.~Z., Wei, C., Kumar, A., and Ma, T. (2022), \enquote{Beyond
  separability: Analyzing the linear transferability of contrastive
  representations to related subpopulations,} \textit{arXiv preprint
  arXiv:2204.02683}.

\bibitem[{He et~al.(2016)He, Zhang, Ren, and Sun}]{he2016deep}
He, K., Zhang, X., Ren, S., and Sun, J. (2016), \enquote{Deep residual learning
  for image recognition,} in \textit{Proceedings of the IEEE conference on
  computer vision and pattern recognition}, pp. 770--778.

\bibitem[{Hendrycks et~al.(2019)Hendrycks, Lee, and
  Mazeika}]{hendrycks2019using}
Hendrycks, D., Lee, K., and Mazeika, M. (2019), \enquote{Using pre-training can
  improve model robustness and uncertainty,} \textit{arXiv preprint
  arXiv:1901.09960}.

\bibitem[{Ho and Nvasconcelos(2020)}]{ho2020contrastive}
Ho, C.-H. and Nvasconcelos, N. (2020), \enquote{Contrastive learning with
  adversarial examples,} \textit{Advances in Neural Information Processing
  Systems}, 33, 17081--17093.

\bibitem[{Hyvarinen et~al.(1998)Hyvarinen, Oja, Hoyer, and
  Hurri}]{hyvarinen1998image}
Hyvarinen, A., Oja, E., Hoyer, P., and Hurri, J. (1998), \enquote{Image feature
  extraction by sparse coding and independent component analysis,} in
  \textit{Proceedings. Fourteenth International Conference on Pattern
  Recognition (Cat. No. 98EX170)}, IEEE, vol.~2, pp. 1268--1273.

\bibitem[{Javanmard and Mehrabi(2021)}]{javanmard2021adversarial}
Javanmard, A. and Mehrabi, M. (2021), \enquote{Adversarial robustness for
  latent models: Revisiting the robust-standard accuracies tradeoff,}
  \textit{arXiv preprint arXiv:2110.11950}.

\bibitem[{Javanmard and Soltanolkotabi(2022)}]{javanmard2022precise}
Javanmard, A. and Soltanolkotabi, M. (2022), \enquote{Precise statistical
  analysis of classification accuracies for adversarial training,} \textit{The
  Annals of Statistics}, 50, 2127--2156.

\bibitem[{Javanmard et~al.(2020)Javanmard, Soltanolkotabi, and
  Hassani}]{javanmard2020precise}
Javanmard, A., Soltanolkotabi, M., and Hassani, H. (2020), \enquote{Precise
  tradeoffs in adversarial training for linear regression,} in
  \textit{Conference on Learning Theory}, PMLR, pp. 2034--2078.

\bibitem[{Jiang et~al.(2020)Jiang, Chen, Chen, and Wang}]{jiang2020robust}
Jiang, Z., Chen, T., Chen, T., and Wang, Z. (2020), \enquote{Robust
  pre-training by adversarial contrastive learning,} \textit{Advances in neural
  information processing systems}, 33, 16199--16210.

\bibitem[{Khosla et~al.(2020)Khosla, Teterwak, Wang, Sarna, Tian, Isola,
  Maschinot, Liu, and Krishnan}]{khosla2020supervised}
Khosla, P., Teterwak, P., Wang, C., Sarna, A., Tian, Y., Isola, P., Maschinot,
  A., Liu, C., and Krishnan, D. (2020), \enquote{Supervised contrastive
  learning,} \textit{arXiv preprint arXiv:2004.11362}.

\bibitem[{Kim et~al.(2020)Kim, Tack, and Hwang}]{kim2020adversarial}
Kim, M., Tack, J., and Hwang, S.~J. (2020), \enquote{Adversarial
  Self-Supervised Contrastive Learning,} in \textit{Advances in Neural
  Information Processing Systems}.

\bibitem[{Li et~al.(2020)Li, Zhou, Xiong, and Hoi}]{li2020prototypical}
Li, J., Zhou, P., Xiong, C., and Hoi, S.~C. (2020), \enquote{Prototypical
  contrastive learning of unsupervised representations,} \textit{arXiv preprint
  arXiv:2005.04966}.

\bibitem[{Mikolov et~al.(2013)Mikolov, Sutskever, Chen, Corrado, and
  Dean}]{mikolov2013distributed}
Mikolov, T., Sutskever, I., Chen, K., Corrado, G.~S., and Dean, J. (2013),
  \enquote{Distributed representations of words and phrases and their
  compositionality,} in \textit{Advances in neural information processing
  systems}, pp. 3111--3119.

\bibitem[{Min et~al.(2020)Min, Chen, and Karbasi}]{min2020curious}
Min, Y., Chen, L., and Karbasi, A. (2020), \enquote{The curious case of
  adversarially robust models: More data can help, double descend, or hurt
  generalization,} \textit{arXiv preprint arXiv:2002.11080}.

\bibitem[{Mo et~al.(2022)Mo, Wu, Wang, Guo, and Wang}]{mo2022adversarial}
Mo, Y., Wu, D., Wang, Y., Guo, Y., and Wang, Y. (2022), \enquote{When
  Adversarial Training Meets Vision Transformers: Recipes from Training to
  Architecture,} \textit{arXiv preprint arXiv:2210.07540}.

\bibitem[{Najafi et~al.(2019)Najafi, Maeda, Koyama, and
  Miyato}]{najafi2019robustness}
Najafi, A., Maeda, S.-i., Koyama, M., and Miyato, T. (2019),
  \enquote{Robustness to adversarial perturbations in learning from incomplete
  data,} in \textit{Advances in Neural Information Processing Systems}, pp.
  5542--5552.

\bibitem[{Nguyen et~al.(2022)Nguyen, Lim, and Torr}]{nguyen2022task}
Nguyen, A.~T., Lim, S.~N., and Torr, P. (2022), \enquote{Task-Agnostic Robust
  Representation Learning,} \textit{arXiv preprint arXiv:2203.07596}.

\bibitem[{Oord et~al.(2018)Oord, Li, and Vinyals}]{oord2018representation}
Oord, A. v.~d., Li, Y., and Vinyals, O. (2018), \enquote{Representation
  learning with contrastive predictive coding,} \textit{arXiv preprint
  arXiv:1807.03748}.

\bibitem[{Petrov and Kwiatkowska(2022)}]{petrov2022robustness}
Petrov, A. and Kwiatkowska, M. (2022), \enquote{Robustness of Unsupervised
  Representation Learning without Labels,} \textit{arXiv preprint
  arXiv:2210.04076}.

\bibitem[{Raghunathan et~al.(2019)Raghunathan, Xie, Yang, Duchi, and
  Liang}]{raghunathan2019adversarial}
Raghunathan, A., Xie, S.~M., Yang, F., Duchi, J.~C., and Liang, P. (2019),
  \enquote{Adversarial training can hurt generalization,} \textit{arXiv
  preprint arXiv:1906.06032}.

\bibitem[{Rice et~al.(2020)Rice, Wong, and Kolter}]{rice2020overfitting}
Rice, L., Wong, E., and Kolter, J.~Z. (2020), \enquote{Overfitting in
  adversarially robust deep learning,} \textit{arXiv preprint
  arXiv:2002.11569}.

\bibitem[{Salman et~al.(2020)Salman, Ilyas, Engstrom, Kapoor, and
  Madry}]{salman2020adversarially}
Salman, H., Ilyas, A., Engstrom, L., Kapoor, A., and Madry, A. (2020),
  \enquote{Do adversarially robust imagenet models transfer better?}
  \textit{arXiv preprint arXiv:2007.08489}.

\bibitem[{Saunshi et~al.(2022)Saunshi, Ash, Goel, Misra, Zhang, Arora, Kakade,
  and Krishnamurthy}]{saunshi2022understanding}
Saunshi, N., Ash, J., Goel, S., Misra, D., Zhang, C., Arora, S., Kakade, S.,
  and Krishnamurthy, A. (2022), \enquote{Understanding contrastive learning
  requires incorporating inductive biases,} \textit{arXiv preprint
  arXiv:2202.14037}.

\bibitem[{Saunshi et~al.(2019)Saunshi, Plevrakis, Arora, Khodak, and
  Khandeparkar}]{saunshi2019theoretical}
Saunshi, N., Plevrakis, O., Arora, S., Khodak, M., and Khandeparkar, H. (2019),
  \enquote{A theoretical analysis of contrastive unsupervised representation
  learning,} in \textit{International Conference on Machine Learning}, PMLR,
  pp. 5628--5637.

\bibitem[{Shafahi et~al.(2019)Shafahi, Saadatpanah, Zhu, Ghiasi, Studer,
  Jacobs, and Goldstein}]{shafahi2019adversarially}
Shafahi, A., Saadatpanah, P., Zhu, C., Ghiasi, A., Studer, C., Jacobs, D., and
  Goldstein, T. (2019), \enquote{Adversarially robust transfer learning,}
  \textit{arXiv preprint arXiv:1905.08232}.

\bibitem[{Shen et~al.(2022)Shen, Jones, Kumar, Xie, HaoChen, Ma, and
  Liang}]{shen2022connect}
Shen, K., Jones, R.~M., Kumar, A., Xie, S.~M., HaoChen, J.~Z., Ma, T., and
  Liang, P. (2022), \enquote{Connect, not collapse: Explaining contrastive
  learning for unsupervised domain adaptation,} in \textit{International
  Conference on Machine Learning}, PMLR, pp. 19847--19878.

\bibitem[{Sinha et~al.(2018)Sinha, Namkoong, and Duchi}]{sinha2018certifying}
Sinha, A., Namkoong, H., and Duchi, J. (2018), \enquote{Certifying some
  distributional robustness with principled adversarial training,} .

\bibitem[{Taheri et~al.(2021)Taheri, Xie, and Lederer}]{taheri2021statistical}
Taheri, M., Xie, F., and Lederer, J. (2021), \enquote{Statistical guarantees
  for regularized neural networks,} \textit{Neural Networks}, 142, 148--161.

\bibitem[{Tian et~al.(2020)Tian, Sun, Poole, Krishnan, Schmid, and
  Isola}]{tian2020makes}
Tian, Y., Sun, C., Poole, B., Krishnan, D., Schmid, C., and Isola, P. (2020),
  \enquote{What makes for good views for contrastive learning?} \textit{arXiv
  preprint arXiv:2005.10243}.

\bibitem[{Tosh et~al.(2021)Tosh, Krishnamurthy, and Hsu}]{tosh2021contrastive}
Tosh, C., Krishnamurthy, A., and Hsu, D. (2021), \enquote{Contrastive learning,
  multi-view redundancy, and linear models,} in \textit{Algorithmic Learning
  Theory}, PMLR, pp. 1179--1206.

\bibitem[{Wang et~al.(2022)Wang, Wang, Zhu, and Wang}]{wang2022improving}
Wang, Q., Wang, Y., Zhu, H., and Wang, Y. (2022), \enquote{Improving
  Out-of-Distribution Generalization by Adversarial Training with Structured
  Priors,} \textit{arXiv preprint arXiv:2210.06807}.

\bibitem[{Wang et~al.(2019{\natexlab{a}})Wang, Ma, Bailey, Yi, Zhou, and
  Gu}]{wang2019convergence}
Wang, Y., Ma, X., Bailey, J., Yi, J., Zhou, B., and Gu, Q.
  (2019{\natexlab{a}}), \enquote{On the convergence and robustness of
  adversarial training,} in \textit{International Conference on Machine
  Learning}, pp. 6586--6595.

\bibitem[{Wang et~al.(2019{\natexlab{b}})Wang, Zou, Yi, Bailey, Ma, and
  Gu}]{wang2019improving}
Wang, Y., Zou, D., Yi, J., Bailey, J., Ma, X., and Gu, Q. (2019{\natexlab{b}}),
  \enquote{Improving adversarial robustness requires revisiting misclassified
  examples,} in \textit{International Conference on Learning Representations}.

\bibitem[{Wen and Li(2021)}]{wen2021toward}
Wen, Z. and Li, Y. (2021), \enquote{Toward understanding the feature learning
  process of self-supervised contrastive learning,} in \textit{International
  Conference on Machine Learning}, PMLR, pp. 11112--11122.

\bibitem[{Wu et~al.(2020{\natexlab{a}})Wu, Chen, Cai, He, and Gu}]{wu2020does}
Wu, B., Chen, J., Cai, D., He, X., and Gu, Q. (2020{\natexlab{a}}),
  \enquote{Does Network Width Really Help Adversarial Robustness?}
  \textit{arXiv preprint arXiv:2010.01279}.

\bibitem[{Wu et~al.(2020{\natexlab{b}})Wu, Wang, and Xia}]{wu2020revisiting}
Wu, D., Wang, Y., and Xia, S.-t. (2020{\natexlab{b}}), \enquote{Adversarial
  Weight Perturbation Helps Robust Generalization,} \textit{arXiv preprint
  arXiv:2004.05884}.

\bibitem[{Xiao et~al.(2021)Xiao, Fan, Sun, and Luo}]{xiao2021adversarial}
Xiao, J., Fan, Y., Sun, R., and Luo, Z.-Q. (2021), \enquote{Adversarial
  Rademacher Complexity of Deep Neural Networks,} .

\bibitem[{Xiao et~al.(2022{\natexlab{a}})Xiao, Fan, Sun, Wang, and
  Luo}]{xiao2022stability}
Xiao, J., Fan, Y., Sun, R., Wang, J., and Luo, Z.-Q. (2022{\natexlab{a}}),
  \enquote{Stability analysis and generalization bounds of adversarial
  training,} \textit{arXiv preprint arXiv:2210.00960}.

\bibitem[{Xiao et~al.(2022{\natexlab{b}})Xiao, Qin, Fan, Wu, Wang, and
  Luo}]{xiao2022adaptive}
Xiao, J., Qin, Z., Fan, Y., Wu, B., Wang, J., and Luo, Z.-Q.
  (2022{\natexlab{b}}), \enquote{Adaptive Smoothness-weighted Adversarial
  Training for Multiple Perturbations with Its Stability Analysis,}
  \textit{arXiv preprint arXiv:2210.00557}.

\bibitem[{Xiao et~al.(2020)Xiao, Wang, Efros, and Darrell}]{xiao2020should}
Xiao, T., Wang, X., Efros, A.~A., and Darrell, T. (2020), \enquote{What should
  not be contrastive in contrastive learning,} \textit{arXiv preprint
  arXiv:2008.05659}.

\bibitem[{Xing et~al.(2021{\natexlab{a}})Xing, Song, and
  Cheng}]{xing2021algorithmic}
Xing, Y., Song, Q., and Cheng, G. (2021{\natexlab{a}}), \enquote{On the
  Algorithmic Stability of Adversarial Training,} \textit{Advances in Neural
  Information Processing Systems}, 34.

\bibitem[{Xing et~al.(2021{\natexlab{b}})Xing, Song, and
  Cheng}]{xing2021generalization}
--- (2021{\natexlab{b}}), \enquote{On the generalization properties of
  adversarial training,} in \textit{International Conference on Artificial
  Intelligence and Statistics}, PMLR, pp. 505--513.

\bibitem[{Yin et~al.(2018)Yin, Ramchandran, and Bartlett}]{yin2018rademacher}
Yin, D., Ramchandran, K., and Bartlett, P. (2018), \enquote{Rademacher
  complexity for adversarially robust generalization,} \textit{arXiv preprint
  arXiv:1810.11914}.

\bibitem[{Zhang et~al.(2019)Zhang, Yu, Jiao, Xing, Ghaoui, and
  Jordan}]{zhang2019theoretically}
Zhang, H., Yu, Y., Jiao, J., Xing, E.~P., Ghaoui, L.~E., and Jordan, M.~I.
  (2019), \enquote{Theoretically Principled Trade-off between Robustness and
  Accuracy,} in \textit{Proceedings of the 36th International Conference on
  Machine Learning}, {PMLR}, vol.~97 of \textit{Proceedings of Machine Learning
  Research}, pp. 7472--7482.

\bibitem[{Zhang et~al.(2021)Zhang, Sang, Yi, Yang, Dong, and Yu}]{zhang2021pre}
Zhang, J., Sang, J., Yi, Q., Yang, Y., Dong, H., and Yu, J. (2021),
  \enquote{Pre-training also Transfers Non-Robustness,} \textit{arXiv preprint
  arXiv:2106.10989}.

\bibitem[{Zhang et~al.(2020{\natexlab{a}})Zhang, Xu, Han, Niu, Cui, Sugiyama,
  and Kankanhalli}]{zhang2020attacks}
Zhang, J., Xu, X., Han, B., Niu, G., Cui, L., Sugiyama, M., and Kankanhalli, M.
  (2020{\natexlab{a}}), \enquote{Attacks which do not kill training make
  adversarial learning stronger,} in \textit{International Conference on
  Machine Learning}, PMLR, pp. 11278--11287.

\bibitem[{Zhang et~al.(2020{\natexlab{b}})Zhang, Plevrakis, Du, Li, Song, and
  Arora}]{zhang2020over}
Zhang, Y., Plevrakis, O., Du, S.~S., Li, X., Song, Z., and Arora, S.
  (2020{\natexlab{b}}), \enquote{Over-parameterized Adversarial Training: An
  Analysis Overcoming the Curse of Dimensionality,} \textit{arXiv preprint
  arXiv:2002.06668}.

\bibitem[{Zhao et~al.(2022)Zhao, Chen, and Du}]{zhao2022blessing}
Zhao, Y., Chen, J., and Du, S.~S. (2022), \enquote{Blessing of Class Diversity
  in Pre-training,} \textit{arXiv preprint arXiv:2209.03447}.

\end{thebibliography}

\onecolumn
\appendix

\section*{Checklist}
 \begin{enumerate}

 \item For all models and algorithms presented, check if you include:
 \begin{enumerate}
   \item A clear description of the mathematical setting, assumptions, algorithm, and/or model. \textbf{Yes}.
   \item An analysis of the properties and complexity (time, space, sample size) of any algorithm. \textbf{Not Applicable}.
   \item (Optional) Anonymized source code, with specification of all dependencies, including external libraries. \textbf{No}, we are using Github repositories from other literature.
 \end{enumerate}

 \item For any theoretical claim, check if you include:
 \begin{enumerate}
   \item Statements of the full set of assumptions of all theoretical results. \textbf{Yes}.
   \item Complete proofs of all theoretical results. \textbf{Yes}.
   \item Clear explanations of any assumptions. \textbf{Yes}. 
 \end{enumerate}

 \item For all figures and tables that present empirical results, check if you include:
 \begin{enumerate}
   \item The code, data, and instructions needed to reproduce the main experimental results (either in the supplemental material or as a URL). \textbf{No}.
   \item All the training details (e.g., data splits, hyperparameters, how they were chosen). \textbf{No}, we are using default settings in the existing code and highlighting the changes in our paper.
         \item A clear definition of the specific measure or statistics and error bars (e.g., with respect to the random seed after running experiments multiple times). \textbf{Yes}.
         \item A description of the computing infrastructure used. (e.g., type of GPUs, internal cluster, or cloud provider). \textbf{No}, we only use a single GPU, and the computation is not expensive.
 \end{enumerate}

 \item If you are using existing assets (e.g., code, data, models) or curating/releasing new assets, check if you include:
 \begin{enumerate}
   \item Citations of the creator If your work uses existing assets. \textbf{Yes}.
   \item The license information of the assets, if applicable. \textbf{Not Applicable}.
   \item New assets either in the supplemental material or as a URL, if applicable. \textbf{Not Applicable}.
   \item Information about consent from data providers/curators. \textbf{Not Applicable}.
   \item Discussion of sensible content if applicable, e.g., personally identifiable information or offensive content. \textbf{Not Applicable}.
 \end{enumerate}

 \item If you used crowdsourcing or conducted research with human subjects, check if you include:
 \begin{enumerate}
   \item The full text of instructions given to participants and screenshots. \textbf{Not Applicable}.
   \item Descriptions of potential participant risks, with links to Institutional Review Board (IRB) approvals if applicable. \textbf{Not Applicable}.
   \item The estimated hourly wage paid to participants and the total amount spent on participant compensation. \textbf{Not Applicable}.
 \end{enumerate}

 \end{enumerate}
\newpage
Below is a list of the contents in this appendix:
\begin{itemize}
    \item Section \ref{sec:sup}: detailed theorem for supervised learning.
    \item Section \ref{sec:dics}: discussion on potential relaxations of the assumptions.
    \item Section \ref{sec:appendix:linf}: discussion when using $\mathcal{L}_{\infty}$ attack.
    \item Section \ref{sec:exp}: real-data experiments.
    \item Section \ref{sec:appendix:simulation}: simulation details.
    \item Section \ref{sec:appendix:proof_l2}: the proof for theorems and lemmas using $\mathcal{L}_2$ attack. 

\end{itemize}

\section{Details for Supervised Learning}\label{sec:sup}

\paragraph{Clean training does not purify features} 
The following theorem indicates that clean training can achieve good clean performance without feature purification. 
\begin{theorem}\label{thm:clean}

For some $(W,b)\in\mathcal{M}$ satisfying $Ua=\theta_0$, for square loss, absolute loss, and logistic regression,
\begin{eqnarray*}
    \mathbb{E}l_0(Z,Y;W,b)=\mathbb{E}l_0(X,Y;\theta_0)+O(\psi),
\end{eqnarray*}
where $\psi$ is a vanishing term induced by the noise $\xi$ and the activation gate, i.e., the discrepancy between $\mathbb{I}(x^{\top}U_{h},0)$ and $\mathbb{I}(x^{\top}U_{h}+\xi^{\top}W_{h},b_h)$. If $Hk^3(m^*)^3=O(d^{2-\varepsilon})$ for some $\varepsilon>0$, then $\psi\rightarrow 0$. 

There are many choices of  $(W,b)\in\mathcal{M}$ with \textbf{good clean performance}, i.e.,
\begin{eqnarray}\label{eq:cleanloss}
    \mathbb{E}l_0(Z,Y;W,b)
    =\min_{(W',b')\in\mathcal{M}}\mathbb{E}l_0(Z,Y;W',b')+O(\psi)=\mathbb{E}l_0(X,Y;\theta_0)+O(\psi).
\end{eqnarray}

Meanwhile, \textbf{their robustness is poor}:
When taking $\epsilon =\Theta( 1/(\log(d)\sqrt{m^*k} ))$,
\begin{eqnarray*}
\mathbb{E}l_\epsilon(Z,Y;W,b)-\min_{(W',b')\in\mathcal{M}}\mathbb{E}l_\epsilon(Z,Y;W',b')=O(\psi)+\Theta( \epsilon \sqrt{m^* k}),
\end{eqnarray*}
The notation $\Theta$ belongs to the family of Big-O notation, and it is the same as $\asymp$. For two sequences $\{a_n\}, \{b_n\}$, $b_n=\Theta(a_n)$ (or $b_n\asymp a_n$) means that when $n\rightarrow\infty$, there exists some constants $c_0,c_1>0$ so that $c_0a_n\leq b_n\leq c_1a_n$.

Note that when $m^*$ and $k$ are small enough, and $H\gg d$ in a suitable range, $\epsilon \sqrt{m^* k}\gg O(\psi)$.

\end{theorem}

The proof of Theorem \ref{thm:clean} and the following Theorem \ref{thm:adv} mainly utilize Lemma \ref{thm:adv_ideal}. In Lemma \ref{thm:adv_ideal}, the results hold in probability. To prove Theorem \ref{thm:clean} and \ref{thm:adv}, the main goal is to quantify the effect when the exceptions happen.

\paragraph{Adversarial training purifies features}
Based on the idea in Lemma \ref{thm:adv_ideal}, the following theorem shows how adversarial training improves robustness and how purification happens. One can purify the neural network using adversarial training while achieving a good performance in both clean and adversarial testing.
\begin{theorem}\label{thm:adv}
Assume $\epsilon=\Theta(1/(\sqrt{km^*}\log d))$ and $H=o(\epsilon d^{3/2})$, then if $W,b\in\mathcal{M}$ leads to \textbf{a small adversarial loss}, i.e.,
\begin{eqnarray*}
\mathbb{E}l_\epsilon(Z,Y;W,b) 
= \min_{W',b'\in\mathcal{M}}\mathbb{E}l_\epsilon(Z,Y;W',b')+O(\psi),
\end{eqnarray*}
then (1) \textbf{its clean performance is also good}:
\begin{eqnarray*}
    \mathbb{E}l_0(Z,Y;W,b) - \mathbb{E}l_0(X,Y;\theta_0)=O(\psi)+O(\epsilon \sqrt{k}),
\end{eqnarray*}
and (2) when $d/H\gg\psi$, $(1-o(1))H$ hidden nodes satisfy $\|U_h\|_0=1$.
\end{theorem}

\section{Potential Relaxations}\label{sec:dics}
Below is a list on the potential relaxations in the theory:

    For the activation function, our choice (\ref{eqn:activation}) simplifies the analysis to highlight the feature purification. For other activation functions, e.g., ReLU, if they can work as a gate to screen out noise, the idea of feature purification still works.

In terms of the architecture of the neural network, under the sparse coding model, as long as the first layer purifies the features, the neural network is always robust to adversarial attacks regardless of the number of layers. Thus, one can extend our analysis to multi-layer neural networks.

For the sparse coding model, this is a key assumption of the feature purification phenomenon. If all features are always active, there is no need to purify them in the hidden nodes to minimize adversarial loss. All features will contribute to the adversarial loss together. For future study, one may consider better connecting sparse models with real data distribution. In addition, one may also relax the linear model assumption between $X$ and $Y$. Intuitively, if the features are not purified, the attacker will attack the weights of the inactive features. Thus from this perspective, feature purification will also work beyond linear data models.


\section{Using \texorpdfstring{$\mathcal{L}_{\infty}$}{Linf} Attack}\label{sec:appendix:linf}
We consider fast gradient sign attack (FGSM) 
$$\delta_\infty = \epsilon \sgn(\partial l/\partial z).$$

We have
    \begin{eqnarray*}
f_{W,b}(z+\delta_\infty)&=& \sigma\left(\left(z+\epsilon\sgn\left(\frac{\partial l}{\partial f}\right)\sgn{(M U\text{diag}(\mathbb{I}(x^{\top}U+\xi^{\top} W,b))a)}\right)^{\top}W,b\right)a\\
&=& \sigma\left(\left(x+\xi^{\top} M+\epsilon\sgn\left(\frac{\partial l}{\partial f}\right)\sgn{(M U\text{diag}(\mathbb{I}(x^{\top}U+\xi^{\top} W,b))a)}\right)^{\top}U,b\right)a\\
&=&\left(x+\xi^{\top} M\right)^{\top}U\text{diag}(\mathbb{I}( (z+\delta_2)^{\top}W ,b))a\\
&&+\epsilon\sgn\left(\frac{\partial l}{\partial f}\right)\sgn{(M U\text{diag}(\mathbb{I}(x^{\top}U+\xi^{\top} W,b))a)}^{\top}MU\text{diag}(\mathbb{I}( (z+\delta_2)^{\top}W ,b))a\\
&=&f_{W,b}(z)+\epsilon\sgn\left(\frac{\partial l}{\partial f}\right)\|M U\text{diag}(\mathbb{I}(x^{\top}U+\xi^{\top} W,b))a\|_1.
\end{eqnarray*}

Assume the first coordinate of $x$ is non-zero. Since with probability tending to 1 (Lemma \ref{lem:2}), all the hidden nodes receiving $x_1$ are activated, we have $$a^{\top}\text{diag}(\mathbb{I}(U^{\top}x+\xi^{\top} W,b))U_{1,:} = a^{\top}U_{1,:}=\theta_1.$$

Assume the second coordinate of $x$ is zero, since we minimized $\|U\|_F$, each non-zero element of $U_{2,:}$ has the same sign as $\theta_2$, and
$$0\leq |a^{\top}\text{diag}(\mathbb{I}(U^{\top}x+\xi^{\top} W,b))U_{2,:}| \leq |\theta_2|,$$
and the left/right equation is satisfied if every node containing $x_2$ is not/is activated.

For $\mathcal{L}_2$ attack, $\|M U\text{diag}(\mathbb{I}(x^{\top}U+\xi^{\top} W,b))a\|_2$ becomes $\|U\text{diag}(\mathbb{I}(x^{\top}U+\xi^{\top} W,b))a\|_2$, so the attack is directly related to the each coordinate of $\theta$.

For $\mathcal{L}_{\infty}$ attack, we want to investigate $\|MD\theta\|_1$ where $D$ is a diagonal matrix with $D_{i,i}=a^{\top}\text{diag}(\mathbb{I}(U^{\top}x+\xi^{\top} W,b))U_{i,:}/\theta_i$. 

One can see that the relationship between $D$ and $\|MD\theta\|_1$ is more complicated because of the existence of $M$. To discuss about $\|MD\theta\|_1$, one need some information about $M$.

\begin{itemize}
    \item Assume $M$ is the identity matrix, then similar to $\mathcal{L}_2$ attack case, we have
    $$ \|\theta_{\mathcal{X}}\|_1\leq \|D\theta\|_1\leq \|\theta\|_1 .$$
    \item If the unitary matrix $M$ satisfies $\|MD\theta\|_1=\Theta(\|D\theta\|_1)$ for all $D,\theta$, then although we cannot claim $\sup \|U_{:,h}\|_0\leq 1$ lead to the minimal $\Delta_{W,b}$, we can still claim that a constant $\sup \|U_{:,h}\|_0$ is preferred than dense mixtures.
\end{itemize}

Under these two cases, all the observations for $\mathcal{L}_2$ apply to $\mathcal{L}_{\infty}$ attack.


\section{Real-Data Experiments}\label{sec:exp}

\subsection{Experiment Results}

Additional results can be found in Table \ref{tab:robust}, Figure \ref{fig:conv1_pretrained_cifar100}, \ref{fig:conv1_pretrained_cifar10_contrastive}, \ref{fig:conv1_pretrained_cifar100_contrastive}. The settings of the experiments can be found in \ref{tab:config}.

\begin{table}[!ht]
\centering

  \begin{tabular}{lccrr}
  \hline Pre-train & Pre-train & Down  & Acc & Robust \\ 
  \hline
  \multirow{6}{*}{CIFAR10} & Clean & Clean  & {0.955} & 0.001 \\ 
  & Clean & Adv Sup  & 0.477 & 0.109 \\ 
   & Adv Sup & -  & 0.810 & 0.495 \\
   & Adv Sup& Clean  & {0.847} & {0.429} \\
   & Adv Sup & Adv Sup  & 0.836 & 0.484 \\
   & Adv Contra& Clean  & {0.831} & {0.393} \\
   & Adv Contra&  Adv Sup & 0.807 & 0.462 \\
  \hline
  \multirow{3}{*}{CIFAR100} & Clean&  Clean & {0.786} & 0.000 \\ 
   & Adv Sup&  Clean & 0.649 & {0.108} \\
   & Adv Contra&  Clean & {0.749} & 0.185 \\\hline
   \multirow{3}{*}{Tiny-Imagenet} & Clean&  Clean & 0.840 & 0.001 \\ 
   & Adv Sup&  Clean & 0.323 & 0.131 \\
   & Adv Contra&  Clean & {0.774} & 0.150 \\

   \hline
\end{tabular}
\caption{Robustness inheritance.}\label{tab:robust}
\end{table}

\begin{sidewaystable}[]
    \centering
    \begin{tabular}{c|c|c}
    \hline
     &  Supervised  & Contrastive   \\
     \hline
     Clean Epochs & 200 & 200 \\
    Adv Epochs & 200 & 1,200 \\
    Downstream Epochs & 200 & 200\\
    Batch Size & 128 & 256  \\
    Transform & Croppring+Horizontal Flipping   &  Croppring+Horizontal Flipping+Color Jittering  \\
    Learning Rate & 0.1 & 0.1  \\
     \multirow{2}{*}{LR Updating Schedule} & \multirow{2}{*}{Divide by 10 after 50\% / 75\% of total epochs}     & Cosine learning rate annealing (SGDR) and   \\
    & & learning rate warmup. Warmup epoch = 10 \\
    Downstream Learning Rate & 0.1 & 0.1 \\
    Attack Iteration & 10 & 20 \\
    Attack Initialization & The original point & Within a random ball near the original point \\
    \hline
    \end{tabular}
    \caption{Detail configurations of different adversarial training experiments. To maximize robustness, the clean training use the ``Supervised'' configuration when pre-training on CIFAR10, and the ``Contrastive'' configuration when pre-training on CIFAR100. }
    \label{tab:config}
\end{sidewaystable}
\begin{figure}[!ht]
    \centering
    \includegraphics[width=0.6\textwidth]{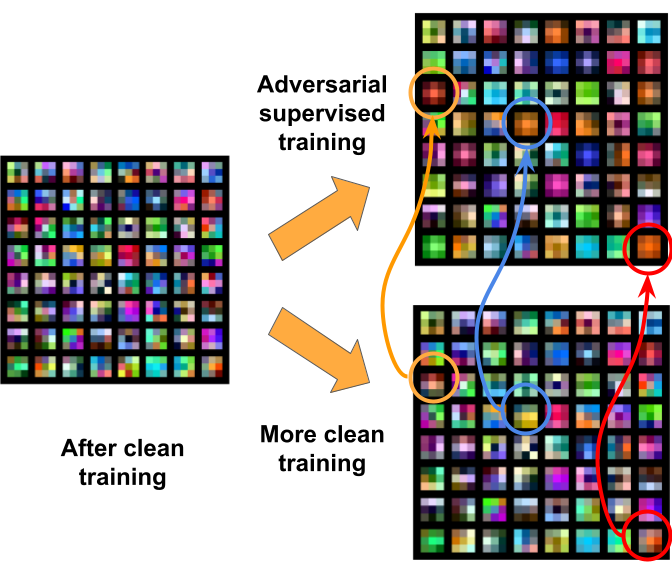}
    \caption{Learned features in the first convolutional layer with and without the adversarial supervised pre-training. The training is performed on CIFAR-100 dataset. Parameters in each filter are normalized to [0,1] separately.}
    \label{fig:conv1_pretrained_cifar100}
\end{figure}


\begin{figure}[!ht]
    \centering
    \includegraphics[width=0.6\textwidth]{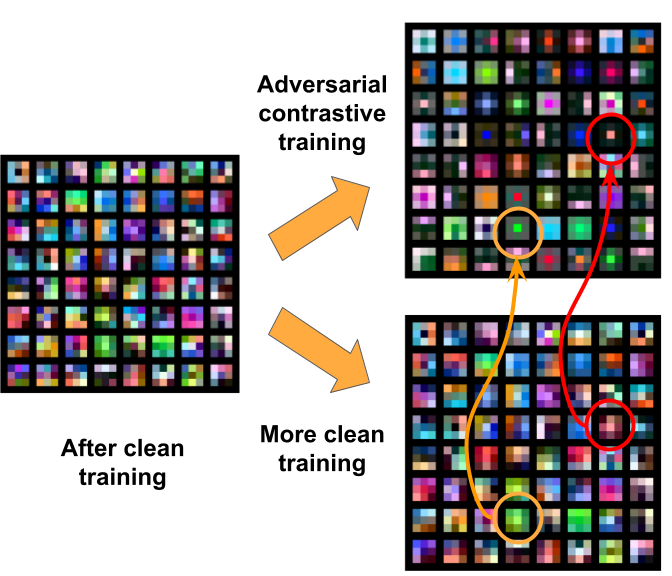}
    \caption{Learned features in the input convolutional layer with and without the adversarial contrastive pre-training. The training is performed on CIFAR-10 dataset. Parameters in each filter are normalized to [0,1] separately. }
    \label{fig:conv1_pretrained_cifar10_contrastive}
\end{figure}

\begin{figure}[!ht]
    \centering
    \includegraphics[width=0.6\textwidth]{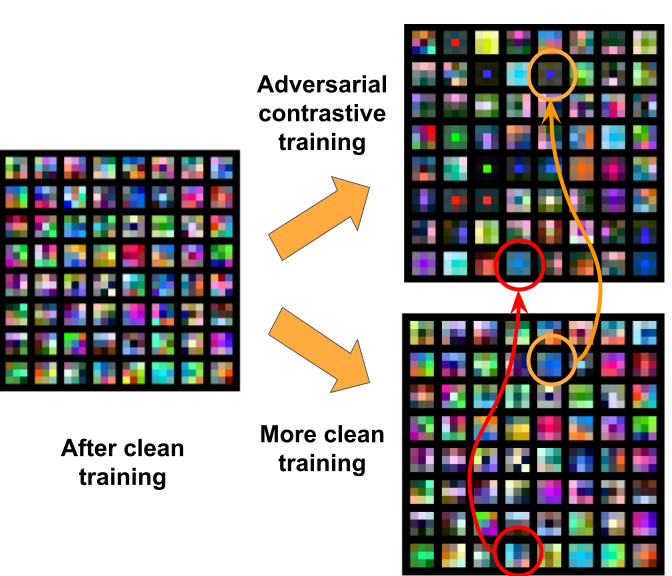}
    \caption{Learned features in the input convolutional layer with and without the adversarial contrastive pre-training. The training is performed on CIFAR-100 dataset. Parameters in each filter are normalized to [0,1] separately. }
    \label{fig:conv1_pretrained_cifar100_contrastive}
\end{figure}

Further, Figure \ref{fig:conv1_pretrained_change} illustrates how the learned features evolve during the contrastive pre-training. 
After 30 epochs of adversarial training, the features show the same purification effects as in supervised learning. This purifying process continues throughout the adversarial training. 

\begin{figure}[!ht]
    \centering
    \includegraphics[width=0.6\textwidth]{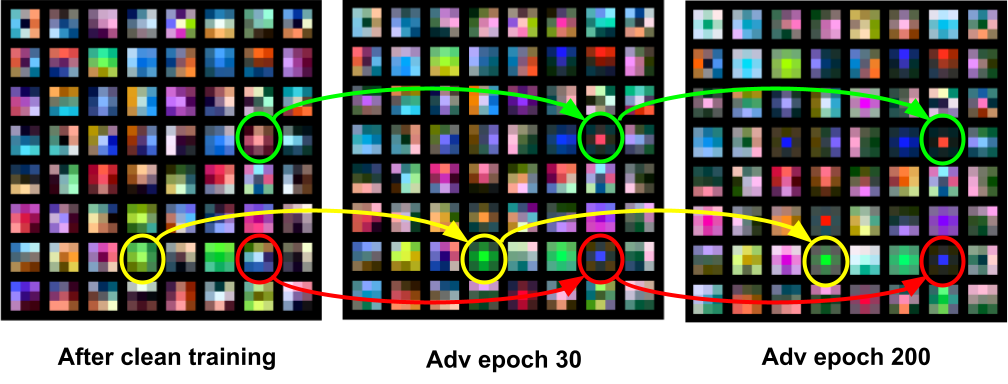}
    \caption{Changes of learned features in the input convolutional layer over adversarial contrastive training on CIFAR-10. Parameters in each filter are normalized to [0,1] separately. }
    \label{fig:conv1_pretrained_change}

\end{figure}

\subsection{Experiment on Resnet-18 with a different input layer}

To better visualize the purification effect in the learned features, we repeat our real data test with an input convolutional layer with a larger kernel size. Specifically, the input layer in this test has Kernel Size = 7, Stride=2, Padding=3. All other layers used the same configuration. We pre-train this modified network using clean training and then adversarial supervised learning on CIFAR-10 dataset. Then we fine-tune the downstream task on CIFAR-10 dataset. 

\begin{table}[ht]
\centering
  \begin{tabular}{lccrr}
  \hline Pre-train Data & Pre-train & Downstream  & Acc & Robust \\ 
  \hline
  \multirow{3}{*}{CIFAR-10} & Clean & Clean  & {0.888} & 0.024 \\ 
   & Adv Sup& Clean  & {0.783} & {0.363} \\
   & Adv Sup& Adv Sup  & 0.789 & 0.401 \\
   \hline
\end{tabular}
\caption{Adversarial robustness and accuracy in CIFAR-10 downstream task. In this test, the input convolutional layer is changed to have Kernel Size = 7, Stride=2, Padding=3.} 
\label{tab:advsup_pretrain_kernel7}

\end{table}

\begin{figure}[!ht]
    \centering
    \includegraphics[width=0.6\textwidth]{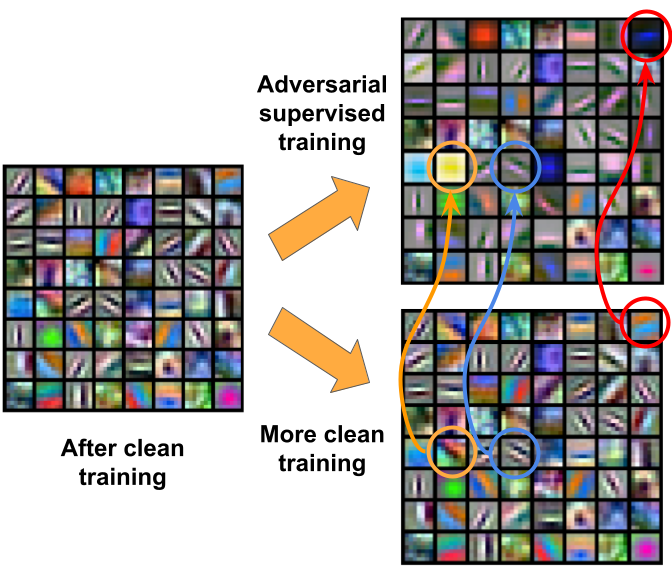}
    \caption{Learned features in the input convolutional layer with and without the adversarial contrastive pre-training. The training is performed on CIFAR-100 dataset. Parameters in each filter are normalized to [0,1] separately. }
    \label{fig:conv1_pretrained_cifar100_kernel7}

\end{figure}

Table \ref{tab:advsup_pretrain_kernel7} shows the standard and robust accuracy. We observe the same inherited robustness in the downstream tasks. Figure \ref{fig:conv1_pretrained_cifar100_kernel7} shows the learned features in the filter. In addition to the reduction of the number of colors, the shapes of features are also simplified, often from multiple parallel lines to one single line, demonstrating purification effects (see blocks marked with blue circles, in which a feature with 4 lines becomes 2 lines.) 

%




\subsection{Effect of Augmentation on Inherited Robustness}\label{exp:aug}

Our primary experiments employ crop\&resize and color distortion augmentations, as highlighted in \cite{wen2021toward}, to enhance feature learning. We evaluate their impact on downstream robustness by comparing test outcomes both with and without these augmentations. As demonstrated in Table~\ref{tab:contratNoAug}, although omitting augmentations diminishes robustness, their presence alone does not substantially improve it. This underscores the pivotal role of adversarial training in achieving robustness.

\begin{table}[ht]
\centering
  \begin{tabular}{lccrr}
  \hline Pre-train Data & Pre-train  & Augmentation & Acc & Robust \\ 
  \hline
  \multirow{5}{*}{CIFAR-10}  & Clean & \checkmark & {0.955} & 0.001 \\ 
   & {Adv Contra} & \checkmark & {0.831}  & {0.393} \\
   & Adv Contra&    $\times$ & 0.822 & 0.389\\
   & {Clean Contra} & \checkmark & 0.897  & 0.004 \\
   & Clean Contra &   $\times$ & 0.889 & 0.001\\
  \hline
  \multirow{5}{*}{CIFAR-100} & Clean  & \checkmark  & {0.786} & 0.000 \\ 
   & {Adv Contra}   & \checkmark  & {0.749} & 0.185 \\
   & Adv Contra&    $\times$ & 0.741 & 0.167 \\
   &  {Clean Contra}  & \checkmark  & 0.801 & 0.005 \\
   & Clean Contra &   $\times$ & 0.797 & 0.000 \\
   \hline
\end{tabular}
\caption{Downstream Task Robustness on Different Augmentation Settings.}
\label{tab:contratNoAug}
\end{table}
\section{Simulation Studies}\label{sec:appendix:simulation}

\subsection{Controlling Feature Purification}
\paragraph{Neural network for supervised learning}
To control the average number of features $m$ in each hidden node, we  
\begin{itemize}
    \item Calculate the number of times each coordinate of $X$ appears: $Hm/d$.
    \item For each coordinate of $X$, we randomly pick $H*m/d$ hidden nodes out of the total $H$ hidden nodes, and take the corresponding elements in $U$ as $d/(Hm)$.
    \item Transform $U$ to $W$ via $W=M^{\top}U $.
    \item Take $b$ as $(\zeta(\log d)/\sqrt{d})(d\sqrt{m}/H)$, where $\zeta(\log d)/\sqrt{d}$ is a probability bound to screen out $\xi$, and $d\sqrt{m}/H$ is the adjustment based on the strength of the features.
\end{itemize}

\paragraph{Neural network for contrastive learning}
To control the average number of features $m$ in contrastive learning, we
\begin{itemize}
    \item Follow the above procedure the generate the hidden layer.
    \item Calculate $A$ as the pseudo inverse of $W$, and take $A=\sqrt{5}A$. 
\end{itemize}

\subsection{Detailed Numbers for Figure  \ref{fig:contra}}

We list all the exact numbers (both average and the corresponding standard error) of Figure \ref{fig:contra} in Table \ref{tab:contra_avg}, \ref{tab:contra_std}, and \ref{tab:alpha}.


\begin{table}[!ht]
\centering
\begin{tabular}{l|llll}
\hline\# Features & Clean Loss, Similar  & Adv Loss, Similar  & Clean Loss, Dissimilar  & Adv Loss, Dissimilar \\\hline
1  & 0.05 & 0.08 & 73.98 & 82.86 \\
2  & 0.05 & 0.08 & 74.11 & 87.87 \\
5  & 0.05 & 0.08 & 74.06 & 91.17 \\
10 & 0.05 & 0.08 & 74.17 & 98.69\\\hline
\end{tabular}
\caption{ The exact average loss ($\times100$) corresponding to Figure \ref{fig:contra}.}\label{tab:contra_avg}
\end{table}

\begin{table}[!ht]
\centering
\begin{tabular}{l|llll}
\hline\# Features & Clean Std, Similar  & Adv Std, Similar  & Clean Std, Dissimilar  & Adv Std, Dissimilar \\\hline
1  & 0.02 & 0.03 & 1.98 & 2.13 \\
2  & 0.02 & 0.03 & 1.92 & 2.04 \\
5  & 0.02 & 0.03 & 1.80 & 1.92 \\
10 & 0.02 & 0.03 & 2.00 & 2.64\\\hline
\end{tabular}
\caption{ The exact standard error ($\times100$) corresponding to Figure \ref{fig:contra}.}\label{tab:contra_std}
\end{table}

\begin{table}[!ht]
\centering
\begin{tabular}{l|llll}
\hline\# Features & Average $\gamma_1$ & Average $\gamma_2$ & Std $\gamma_1$ & Std $\gamma_2$\\\hline
1 & 8.86E-04 & 5.39E-06 & 8.56E-05 & 1.81E-06 \\
2 & 1.87E-03 & 7.42E-06 & 1.15E-04 & 2.42E-06 \\
5 & 4.62E-03 & 1.05E-05 & 1.32E-04 & 3.52E-06 \\
10 & 8.87E-03 & 1.51E-05 & 3.32E-04 & 5.52E-06\\\hline
\end{tabular}
\caption{The average and standard error of  $\gamma_1$ and $\gamma_2$. }\label{tab:alpha}
\end{table}

\section{Proof for \texorpdfstring{$\mathcal{L}_2$}{L2} Attack}\label{sec:appendix:proof_l2}
In this section, we present the proofs using $\mathcal{L}_2$ adversarial training for all the theorems and lemmas in Section \ref{sec:intuition}, \ref{sec:pretrain_super}, \ref{sec:contra}, and \ref{sec:downstream}. 
\subsection{Some Lemmas and Probability Bounds}
\begin{lemma}\label{lem:1}
Denote a vector $m\in\mathbb{R}^d$ with $\|m\|_2=1$ and a random vector $\xi\sim N(0,I_d/d)$, then for any $t>1$,
\begin{eqnarray}
P\left( |m^{\top}\xi|>t\sqrt{\frac{1}{d}} \right)\leq \sqrt{\frac{2}{\pi}} \frac{1}{t}\exp(-t^2/2),
\end{eqnarray}
and
\begin{eqnarray}
\mathbb{E}\left[ |m^{\top}\xi|\bigg| |m^{\top}\xi|>t\sqrt{\frac{1}{d}} \right] = \frac{2}{\sqrt{d}} \frac{\phi(t)}{\Phi(t)}.
\end{eqnarray}
\end{lemma}
\begin{proof}[Proof of Lemma \ref{lem:1}]
Observe that $\sqrt{d}m^{\top}\xi$ follows $N(0,1)$. One can directly bound $|m^{\top}\xi|$ using Gaussian tail bound.
\end{proof}

\subsection{Proof for Section \ref{sec:intuition}}
To prove Lemma \ref{thm:adv_ideal}, we show the following things:
\begin{itemize}
    \item The distribution of $X^{\top}U_h$. (Lemma \ref{lem:lem})
    \item Without attack, whether the hidden nodes are activated as long as corresponding features are non-zero. (Lemma \ref{lem:2})
    \item Further adding attack, whether the nodes are not activated/deactivated additionally by the attack. (Lemma \ref{lem:3})
    \item After proving Lemma \ref{lem:lem}, \ref{lem:2}, and \ref{lem:3}, we finally present the proof of Lemma \ref{thm:adv_ideal}.
\end{itemize}

\begin{lemma}\label{lem:lem}
Consider the case where $\zeta=0$, i.e., $Z$ is a linear transformation of $X$. Denote $m= \sup_h\|U_h\|_0\leq m^*$. Given at least one feature received by the hidden node $h$ is non-zero, for $v=o(\|U_h\|_{\infty}/\sqrt{k})$, the conditional distribution of $X^{\top}U_{h}$ satisfies
\begin{eqnarray*}
&&P\left( |X^{\top}U_{h}|<v\mid X_1=x_1, |U_{1,h}|>0 \right)\\
&=&\begin{cases}
0 & m=1\\
O\left(\frac{vk^{3/2}}{d\|U_h\|}+\left(\frac{k}{d}\right)^2\right) & m=\Theta(1)\\
\Phi\left( \frac{v-U_{1,h}x_1}{\sqrt{\sigma^2\|U_{-1,h}\|^2/d }} \right)+ c_u'\frac{1}{1+\left|\frac{v-U_{1,h}x_1}{\sqrt{\sigma^2\|U_{-1,h}\|^2/d }}\right|^3}\sqrt{\frac{d}{mk}}&m\rightarrow\infty
\end{cases}.
\end{eqnarray*}
As a result,
\begin{eqnarray}
&&P(\exists h=1,\ldots,H,\; s.t. |X^{\top}U_{h}|\in(0,v/\|U_h\|)\text{ and }\exists |X_iU_{i,h}|>0)\label{eqn:bound}\\
&=&\begin{cases}
0 & m=1\\
O\left(vH\frac{k^{5/2}}{d^2}+H\left(\frac{k}{d}\right)^3\right) & m=\Theta(1)\\
O\left(\frac{mkH}{d}\Phi\left( \frac{v\|U_h\|-U_{1,h}/\sqrt{k}}{\sqrt{\sigma^2\|U_{-1,h}\|^2/d }} \right)+ c_u'\frac{H}{1+\left|\frac{v\|U_h\|-U_{1,h}/\sqrt{k}}{\sqrt{\sigma^2\|U_{-1,h}\|^2/d }}\right|^3}\sqrt{\frac{mk}{d}}\right)& m\rightarrow\infty
\end{cases},\nonumber
\end{eqnarray}
and when taking $v$ such that $v/\|U_h\| \gg 1/\sqrt{d}$ and $v=o(1/\sqrt{mk})$, if $H=o(d^2/(k^2m^3))$ when $m\rightarrow\infty$, 
$$P(\exists h=1,\ldots,H,\; s.t. |X^{\top}U_{h}|\in(0,v/\|U_h\|)\text{ and }\exists |X_iU_{i,h}|>0)\rightarrow 0.$$
\end{lemma}
\begin{proof}
We consider three regimes: (1) $\|U_h\|_{0}=1$, (2) $\|U_h\|_0=m$ for some constant $m$, and (3) $\|U_h\|_0\rightarrow\infty$.

\paragraph{Case 1, $\|U_h\|_{0}=1$} If a node only contains a single feature, then the conditional distribution of $X^{\top}U_h$ given the feature is non-zero is a single value.
\paragraph{Case 2, $\|U_h\|_{0}=m$} If there are $m$ features for some constant $m$, then the probability that all features are zero is $(1-k/d)^m$, and the probability that only one of the features is non-zero is in $O(mk/d)$.

Given at least one feature is non-zero in the node, the probability of two or more features being non-zero is $1-(1-k/d)^{m-1}$, and the probability of exactly two features being non-zero is $(m-1)(1-k/d)^{m-2}(k/d)$. 

When there are two features activated, denoting $\mathcal{X}$ as the set of non-zero features, under Assumption \ref{assumption:nn}, the probability of $|X^{\top}U_h|<v$ for $v=o(\|(U_h)_{\mathcal{X}}\|)$ is in $O(v\sqrt{k}/\|U_h\|)$.

As a result, for $v=o(\|U_h\|)$, since $m$ is a constant, we have
\begin{eqnarray*}
&&P(|X^{\top}U_h|<v\mid\text{At least one feature is non-zero})\\
&\leq&P(\text{ Two features are non-zero }, |X^{\top}U_h|<v\mid\text{At least one feature is non-zero})\\
&&+P(\text{ Three or more features are non-zero}\mid\text{At least one feature is non-zero})\\
&=& O\left(\frac{v\sqrt{k}}{\|U_h\|}(m-1)\frac{k}{d}\left(1-\frac{k}{d}\right)^{m-2}\right)+1-\left(1-\frac{k}{d}\right)^{m-1}-(m-1)\frac{k}{d}\left(1-\frac{k}{d}\right)^{m-2}\\
&=&O\left(\frac{vk^{3/2}}{d\|U_h\|}+\left(\frac{k}{d}\right)^2\right).
\end{eqnarray*}
\paragraph{Case 3, $\|U_h\|_{0}\rightarrow\infty$}
If there are $m\rightarrow\infty$ features and $m=o(d/k)$, then assuming the first coordinate $X_1=x_1$ is nonzero and $U_{1,h} \neq 0$, $U_h^{\top}X$ conditionally approximately follows a normal distribution $N(U_{1,h}x_1, \sigma^2\|U_{-1,h}\|^2/d )$. In this case, denoting $\Phi$ and the probability function of the standard Gaussian distribution, using non-uniform Berry-Esseen bound in \cite{grigor2012upper}, for some universal constant $c_u$,
\begin{eqnarray*}
&&P(X^{\top}U_h<v\mid  X_1=x_1, U_{1,h}x_1>0)\\
&\leq& \Phi\left( \frac{v-U_{1,h}x_1}{\sqrt{\sigma^2\|U_{-1,h}\|^2/d }} \right)+ c_u\frac{\sum \mathbb{E}|X_iU_{i,h}|^3}{\left(1+\left|\frac{v-U_{1,h}x_1}{\sqrt{\sigma^2\|U_{-1,h}\|^2/d }}\right|^3\right)}\frac{1}{( \sum \mathbb{E}|X_iU_{i,h}|^2 )^{3/2}},
\end{eqnarray*}
where
\begin{eqnarray*}
&&\sum_{i>1}\mathbb{E}|X_iU_{i,h}|^2 = \sigma^2\|U_{-1,h}\|^2/d,\\
&&\sum_{i>1}\mathbb{E}|X_iU_{i,h}|^3 = O\left(\sum_{i>1} |U_{i,h}|^3\frac{1}{d\sqrt{k}} \right).
\end{eqnarray*}
Under Assumption \ref{assumption:nn}, we have
$$\sum_{i>1} |U_{i,h}|^3=O\left(m \left(\frac{\|U_{-1,h}\|^2}{m} \right)^{3/2}\right)=O\left(\frac{\|U_{-1,h}\|^3}{\sqrt{m}}\right).$$
As a result, we conclude that for some constant $c_u'>0$,
\begin{eqnarray*}
&&P(X^{\top}U_h<v\mid  X_1=x_1, U_{1,h}x_1>0)\\
&\leq& \Phi\left( \frac{v-U_{1,h}x_1}{\sqrt{\sigma^2\|U_{-1,h}\|^2/d }} \right)+ c_u'\frac{ \|U_{-1,h}\|^3/(d\sqrt{km})}{\left(1+\left|\frac{v-U_{1,h}x_1}{\sqrt{\sigma^2\|U_{-1,h}\|^2/d }}\right|^3\right)}\frac{1}{(\|U_{-1,h}\|^2/d)^{3/2}}\\
&=&\Phi\left( \frac{v-U_{1,h}x_1}{\sqrt{\sigma^2\|U_{-1,h}\|^2/d }} \right)+ c_u'\frac{1}{\left(1+\left|\frac{v-U_{1,h}x_1}{\sqrt{\sigma^2\|U_{-1,h}\|^2/d }}\right|^3\right)}\sqrt{\frac{d}{mk}}.
\end{eqnarray*}
We use Berry-Esseen bound rather than Hoeffding/Berstein bounds because the latter ones involve the range $M=\sup |X^{\top}U_h|-\inf |X^{\top}U_h|$ which is too broad.

The final (\ref{eqn:bound}) is a union bound taken for the $m$ features in each node for all $H$ hidden nodes.
\end{proof}
After discussing the distribution of $X^{\top}U_h$, we further add the noise $\xi$ (Lemma \ref{lem:2}) and the attack (Lemma \ref{lem:3}) into the model.
\begin{lemma}\label{lem:2}
For any $v\geq 1/\sqrt{d}$,
\begin{eqnarray*}
P\left(\sup_h |\xi^{\top}W_h|/\|W_h\| >v \right)=O\left( \frac{H\sqrt{d}}{v}\exp(-v^2 d/(2\zeta^2)) \right).
\end{eqnarray*}
When taking $v\gg \sqrt{(\log d)/d}$ and $H=\text{poly}(d)$, the probability bound goes to zero.

Furthermore, under the conditions of Lemma \ref{lem:lem}, when taking $b_h$ such that $b_h\gg \sqrt{(\log d)/d} \|U_h\|$ and $b_h=o(1/\sqrt{k\sup_h\|U_h\|_0})$,
$$P(\exists h=1,\ldots, H,\; s.t. \text{ Node }h\text{ is activated by }\xi\text{ or deactivated})\rightarrow 0.$$
\end{lemma}
\begin{proof}[Proof of Lemma \ref{lem:2}]
Since $\xi\sim N(\textbf{0},\zeta^2 I_d/d)$, we have $\zeta^{\top}W_h/\|W_h\|\sim N(0,\zeta^2/d)$. From Lemma \ref{lem:1}, we obtain
\begin{eqnarray}\label{eqn:lem:2:bound}
P\left(\sup_h |\xi^{\top}W_h|/\|W_h\| >v \right)\leq HP(|\xi^{\top}W_h|/\|W_h\| >v )=O\left( \frac{H\sqrt{d}}{v}\exp(-v^2 d/(2\zeta^2)) \right),
\end{eqnarray}
which provides the tail distribution of $|\xi^{\top}W_h|/\|W_h\|$. When taking $v\gg\sqrt{(\log d)/d}$, the above probability goes to zero, i.e., with probability tending to 1, no extra nodes are activated due to $\xi$.

Furthermore, for the nodes which are activated by non-zero features, we have
\begin{eqnarray*}
P(\exists h=1,\ldots,H,\; s.t.,\;\text{Node }h\text{ is deactivated})=O(H P(\text{Node }h\text{ is deactivated})),
\end{eqnarray*}
and
\begin{eqnarray*}
&& P(\text{Node }h\text{ is deactivated})\\
&\leq& P(\text{Node }h\text{ is deactivated}\mid \text{One feature in node }h\text{ is non-zero})\\
&&\qquad\qquad\qquad\qquad\qquad\qquad\qquad\qquad\qquad\times P(\text{One feature in node }h\text{ is non-zero})\\
&&+P(\text{Two or more features in node }h\text{ is non-zero}),
\end{eqnarray*}
where taking $m=\|U_h\|_0$,
\begin{eqnarray*}
&&P(\text{Node }h\text{ is deactivated}\mid \text{One feature in node }h\text{ is non-zero})\\
&=&O\left( m P(|X^{\top}U_h|< b_h+|\xi^{\top}U_h|\mid X_1=x_1, U_{1,h}x_1>0) \right)\\
&=&O\left( m P(|X^{\top}U_h|< 2 b_h\mid X_1=x_1, U_{1,h}x_1>0) \right)+O\left( m P(|\xi^{\top}U_h|>b_h) \right).
\end{eqnarray*}
Therefore, taking $b_h\gg \sqrt{(\log d)/d}\|U_h\|$ and $b_h=o(\|U_h\|/\sqrt{k\sup_h\|U_h\|_0})$, when taking proper $H$, one can show that 
\begin{eqnarray*}
&&P(\exists h=1,\ldots, H,\; s.t. \text{ Node }h\text{ is activated by }\xi\text{ or deactivated})\\
&=&P(\exists h=1,\ldots,H,\; s.t.,\;\text{Node }h\text{ is activated by }\xi)+P(\exists h=1,\ldots,H,\; s.t.,\;\text{Node }h\text{ is deactivated})\\
&\leq& O\left( \frac{H\sqrt{d}\|U_h\|}{b_h}\exp(-b_h^2 d/(2\zeta^2\|U_h\|^2)) \right)\\
&&+O\left( Hm P(|X^{\top}U_h|< 2 b_h\mid X_1=x_1, U_{1,h}x_1>0) \right)+ O\left( \frac{Hm\sqrt{d}\|U_h\|}{b_h}\exp(-b_h^2 d/(2\zeta^2\|U_h\|^2)) \right)\\
&&+O(H P(\text{Two or more features in node }h\text{ is non-zero in node }h)),
\end{eqnarray*}
where all the four terms go to zero based on Lemma \ref{lem:lem} and (\ref{eqn:lem:2:bound}).
\end{proof}

\begin{lemma}\label{lem:3}
Under the conditions of Lemma \ref{lem:lem}, when taking $b_h$ such that $b_h\gg \sqrt{(\log d)/d} \|U_h\|$ and $b_h=o(1/\sqrt{k\sup_h\|U_h\|_0})$, and $\epsilon=o(\inf_h b_h/\|U_h\|)$ for $\mathcal{L}_{2}$ attack,
$$P(\exists h=1,\ldots, H,\; s.t. \text{ Node }h\text{ is activated by }\xi\text{ and the attack or deactivated})\rightarrow 0.$$
For $\mathcal{L}_{\infty}$ attack, when $\epsilon=o( \inf_h b_h/\|U_h\|/\sqrt{\|U_h\|_0} )$, the above inequality also holds.
\end{lemma}
\begin{proof}[Proof of Lemma \ref{lem:3}]
Since $\epsilon=o(\inf_h b_h/\|U_h\|)$, if all the features in node $h$ are zero, then we have
\begin{eqnarray*}
P\left(\sup_h |\xi^{\top}W_h|/\|W_h\|+\epsilon >v \right)=O\left( \frac{H\sqrt{d}}{v-\epsilon}\exp(-(v-\epsilon)^2 d/(2\zeta^2)) \right)=O\left( \frac{H\sqrt{d}}{v}\exp(-v^2 d/(2\zeta^2)) \right),
\end{eqnarray*}
Thus with probability tending to 1, all the nodes will not be additionally activated by $\xi$ and the attack.

Furthermore, if a node $h$ is activated by non-zero features, then we have the following decomposition:
\begin{eqnarray*}
&& P(\text{Node }h\text{ is deactivated})\\
&\leq& P(\text{Node }h\text{ is deactivated}\mid \text{One feature in node }h\text{ is non-zero})\\
&&\qquad\qquad\qquad\qquad\qquad\qquad\qquad\qquad\qquad\times P(\text{One feature in node }h\text{ is non-zero})\\
&&+P(\text{Two or more features in node }h\text{ is non-zero}),
\end{eqnarray*}
where taking $m=\|U_h\|_0$,
\begin{eqnarray*}
&&P(\text{Node }h\text{ is deactivated}\mid \text{One feature in node }h\text{ is non-zero})\\
&=&O\left( m P(|X^{\top}U_h|< b_h+\epsilon\|U_h\|+|\xi^{\top}U_h|\mid X_1=x_1, U_{1,h}x_1>0) \right)\\
&=&O\left( m P(|X^{\top}U_h|< 2 b_h\mid X_1=x_1, U_{1,h}x_1>0) \right)\\
&&+O\left( m P(\epsilon\|U_h\|+|\xi^{\top}U_h|>b_h) \right),
\end{eqnarray*}
and the final steps are the same as in Lemma \ref{lem:2}.

For $\mathcal{L}_{\infty}$ attack, $|\delta_{\infty}^{\top}U_h|=\epsilon\|U_h\|_1=O(\epsilon\|U_h\|\sqrt{m}) $, and one can replace the $\epsilon$ in the derivations of $\mathcal{L}_2$ attack case with $\epsilon\sqrt{m}$ to go through the proof.
\end{proof}
\begin{proof}[Proof of Lemma \ref{thm:adv_ideal}]
For $\mathcal{L}_2$ attack, since we consider the FGM attack, we first calculate the gradient of $l$ w.r.t. $z$. Denote $f_{W,b}$ as the non-linear neural network, then 
\begin{eqnarray*}
\frac{\partial}{\partial z} l_0(z,y;W,b) = \frac{\partial l_0}{\partial f} \frac{\partial f}{\partial z},
\end{eqnarray*}
where
\begin{eqnarray*}
\frac{\partial f}{\partial z} = \frac{\partial}{\partial z} \sigma( z^{\top}W,b )a= W\text{diag}(\mathbb{I}(z^{\top}W,b))a=M U\text{diag}(\mathbb{I}(x^{\top}U+\xi^{\top} W,b))a.
\end{eqnarray*}
As a result, the attack becomes
\begin{eqnarray*}
\delta_2(z,y,f_{W,b},l) = \epsilon\sgn\left(\frac{\partial l}{\partial f}\right)\frac{M U\text{diag}(\mathbb{I}(x^{\top}U+\xi^{\top} W,b))a}{\|M U\text{diag}(\mathbb{I}(x^{\top}U+\xi^{\top} W,b))a\|},
\end{eqnarray*}
and
\begin{eqnarray*}
\delta_2(z,y,f_{W,b},l) = \epsilon\sgn\left(\frac{\partial l}{\partial f}\right)\text{sign}{\left[M U\text{diag}(\mathbb{I}(x^{\top}U+\xi^{\top} W,b))a\right]}.
\end{eqnarray*}
Using $\mathcal{L}_2$ attack, the attacked fitted value becomes
\begin{eqnarray*}
f_{W,b}(z+\delta_2)&=& \sigma\left(\left(z+\epsilon\sgn\left(\frac{\partial l}{\partial f}\right)\frac{M U\text{diag}(\mathbb{I}(x^{\top}U+\xi^{\top} W,b))a}{\|M U\text{diag}(\mathbb{I}(x^{\top}U+\xi^{\top} W,b))a\|}\right)^{\top}W,b\right)a\\
&=& \sigma\left(\left(x+\xi^{\top} M+\epsilon\sgn\left(\frac{\partial l}{\partial f}\right)\frac{ U\text{diag}(\mathbb{I}(x^{\top}U+\xi^{\top} W,b))a}{\| U\text{diag}(\mathbb{I}(x^{\top}U+\xi^{\top} W,b))a\|}\right)^{\top}U,b\right)a\\
&=&\left(x+\xi^{\top} M+\epsilon\sgn\left(\frac{\partial l}{\partial f}\right)\frac{ U\text{diag}(\mathbb{I}(x^{\top}U+\xi^{\top} M,b))a}{\| U\text{diag}(\mathbb{I}(x^{\top}U+\xi^{\top} W,b))a\|}\right)^{\top}U\text{diag}(\mathbb{I}( (z+\delta_2)^{\top}W ,b))a.
\end{eqnarray*}
In order to cancel some terms in the above representation, we need that $\mathbb{I}( (z+\delta_2)^{\top}W ,b)=\mathbb{I}( z^{\top}W ,b)$ in probability when $\epsilon=o(b^*)$, which has been shown in Lemma \ref{lem:3}.

As a result,  with probability tending to 1,
\begin{eqnarray*}
f_{W,b}(z+\delta_2)=
f_{W,b}(z)+\epsilon\sgn\left(\frac{\partial l}{\partial f}\right)\| U\text{diag}(\mathbb{I}(x^{\top}U+\xi^{\top} W,b))a\|.
\end{eqnarray*}

Thus we have
\begin{eqnarray*}
\Delta_{W,b}(z,y)&=&l_\epsilon(z,y; W,b)-l_0(z,y; W,b)\\
&=&l_0(z+\delta_2,y; W,b)-l_0(z,y; W,b)\\
&=&\frac{\partial l_0}{\partial f_{w,b}}(f_{W,b}(z+\delta_2)-
f_{W,b}(z)) + O(\epsilon^2) \\
&=&
\epsilon\frac{\partial l}{\partial f_{W,b}}\left\|a^{\top}\text{diag}(\mathbb{I}(W^{\top}z ,b)) W^{\top}\right\|_2+o,
\end{eqnarray*}
where $o$ represents the remainder term.

Assume the first coordinate of $x$ is non-zero. Since with probability tending to 1 (Lemma \ref{lem:2}), all the hidden nodes receiving $x_1$ are activated, we have $$a^{\top}\text{diag}(\mathbb{I}(U^{\top}x+\xi^{\top} W,b))U_{1,:} = a^{\top}U_{1,:}=\theta_1.$$

Assume the second coordinate $x_2$ of $x$ is zero, when $(W,b)\in\mathcal{M}$, all non-zero elements in $U_{2,:}$ have the same sign as $\theta_2$, and
$$0\leq |a^{\top}\text{diag}(\mathbb{I}(U^{\top}x+\xi^{\top} W,b))U_{2,:}| \leq |\theta_2|,$$
and the left/right equation holds if every node containing $x_2$ is not/is activated.

As a result, we conclude that, with probability tending to 1,
$$\|\theta_{\mathcal{X}}\|_2\leq \|a^{\top}\text{diag}(\mathbb{I}(U^{\top}X+\xi^{\top} W,b))U^{\top}\|_2\leq\|\theta\|_2.$$


\end{proof}

\subsection{Proof for Supervised Pre-training}

Denote
\begin{eqnarray}\label{eqn:psi}
\psi=\frac{Hm^3k^{3}\log^2 k}{d^2}+\sqrt{\frac{k}{d}}.
\end{eqnarray}

Before we start the proof of the theorems, we provide an additional lemma to characterize $X^{\top}\theta_0$. Different from the results in Section \ref{sec:intuition}, since we directly work on the risk, rather than probability bounds, we need to know the distribution of $X^{\top}\theta_0$.
\begin{lemma}\label{lem:4}
Under Assumption \ref{assumption:x}, 
\begin{eqnarray*}
P(|X^{\top}\theta_0-\mathbb{E}(X^{\top}\theta_0)|>v)=O\left(\Phi(-v)\right)+O\left(\frac{1}{\sqrt{k}(1+v^3)} \right).
\end{eqnarray*}
\end{lemma}
\begin{proof}[Proof of Lemma \ref{lem:4}]
Using Berry-Esseen bound, we have
\begin{eqnarray*}
&&P(X^{\top}\theta_0-\mathbb{E}(X^{\top}\theta_0)<v)\\
&\leq& \Phi\left( \frac{v-0}{\sqrt{\sigma^2\|\theta_0\|^2/d }} \right)+ c_u\frac{\sum \mathbb{E}|X_i\theta_i|^3}{\left(1+\left|\frac{v}{\sqrt{\sigma^2\|\theta_0\|^2/d }}\right|^3\right)}\frac{1}{( \sigma^2\|\theta_0\|^2/d )^{3/2}},
\end{eqnarray*}
where
\begin{eqnarray*}
\sum \mathbb{E}|X_i\theta_i|^3=\Theta\left(d\frac{k}{d}\frac{1}{k^{3/2}}\right)=\Theta\left(\frac{1}{\sqrt{k}}\right).
\end{eqnarray*}
As a result, for $v<0$,
\begin{eqnarray*}
P(X^{\top}\theta_0-\mathbb{E}(X^{\top}\theta_0)<v)=O\left(\Phi(v)\right)+O\left(\frac{1}{\sqrt{k}(1+v^3)} \right).
\end{eqnarray*}
Similarly, we also have for $v>0$,
\begin{eqnarray*}
P(X^{\top}\theta_0-\mathbb{E}(X^{\top}\theta_0)>v)=O\left(\Phi(-v)\right)+O\left(\frac{1}{\sqrt{k}(1+v^3)} \right),
\end{eqnarray*}
and merging the two sides we have
\begin{eqnarray*}
P(|X^{\top}\theta_0-\mathbb{E}(X^{\top}\theta_0)|>v)=O\left(\Phi(-v)\right)+O\left(\frac{1}{\sqrt{k}(1+v^3)} \right).
\end{eqnarray*}
\end{proof}

\begin{proof}[Proof of Theorem \ref{thm:clean}]

The proof idea is to design some $(W,b)$ such that $(W,b)\in\mathcal{M}$ with good clean performance but poor adversarial performance.

There are two claims for clean performance in Theorem \ref{thm:clean}. In the proof, we merge the two proofs together, and finally figure out what is $\psi$.

\paragraph{Construction} 
Assume in each node, there are $m$ learned features. Here we can take any $m\leq m^*$. Then since there are $H$ hidden nodes, every feature appears in $Hm/d$ nodes. 

To design a $(W,b)$, we split the total $d$ features into groups of $m$ features. For each group, we randomly pick $Hm/d$ hidden nodes and assign non-zero weights for all features in this group. Thus $\|U_h\|=\Theta(d/(H\sqrt{m}))$. The intercept $b_h$ can be determined correspondingly.

Through the above construction, one can obtain a neural network with good clean performance. The adversarial performance is related to $m$.

\paragraph{Clean performance}We first analyze how the noise $\xi$ affects the performance.

Based on Lemma \ref{lem:2}, when taking all $b_h$ as the same value $b_h= t \zeta\sqrt{(\log d)/d}\|U_h\|=t\zeta\sqrt{d\log d}/(H\sqrt{m})$ such that $t^2/2>1$, we have
\begin{eqnarray*}
P\left( \sup_h |U_h^{\top}\xi|>b_h \right)=O\left( \frac{H}{d^{t^2/2}}\right).
\end{eqnarray*}

Besides the noise $\xi$, another error is caused by the event that the hidden nodes are deactivated when more than one feature is active. Based on Lemma \ref{lem:lem}, since we are constructing neural networks whose $m\rightarrow\infty$, we have
$$
P(\exists h=1,\ldots, H,\; s.t.\;|X^{\top}U_h|\in(0,b_h) \text{ while }\exists |X_iU_{i,h}|>0)=O\left(\frac{k^2m^2 H}{d^2} \right),
$$
while we want the above union bound to be small enough, we also require $km/d\rightarrow 0$ so that for each node, the probability goes to zero.

Denote $E(h)=1\{ |\xi^{\top}W_h|>b_h\text{ or }(|X^{\top}U_h|\in(0,2b_h)\text{ while }\exists |X_iU_{i,h}|>0 )\}$. If node $h$ is activated/deactivated by noise or the non-zero features cancel with each other, we always have $E(h)=1$. As a result, we use $E(h)$ as the upper bound of the event $\mathbb{I}(Z^{\top}W_h, b_h)\neq \mathbb{I}(X^{\top}U_h, 0)$. When taking $t$ large enough, we have
$$
P(\exists h=1,\ldots, H,\; s.t.\; E(h)=1)=O\left(\frac{k^2m^2 H}{d^2} \right)+O\left( \frac{H}{d^{t^2/2}}\right)=O\left(\frac{k^2m^2 H}{d^2} \right).
$$

For each hidden node, if $E(h)=1$, then it leads to at most $b_h$ of error. For square loss, when there is only one or several nodes activated/deactivated by the noise, then there will only be $\Theta(b_h)$ error in the fitted value, which is negligible and leads to $O(b_h)$ increase in loss. In the worst case, when all hidden nodes are activated/deactivated by the noise, the fitted value could involve $\Delta$ error, leading to an increase of $(\Delta+Y)^2-Y^2=\Delta (\Delta+2Y)$ in the loss.

Since each hidden node has at most $m$ features and $m\ll d$, when there are $k$ features are nonzero and $E(h)=1$ for some $h$, there are $O(kHm/d)$ hidden nodes which are mistakenly deactivated in the worst case. We ignore the effect of the noise $\xi$ because it is negligible.  

We have
\begin{eqnarray*}
&&\mathbb{E}l_0(Z,Y;W,b)\\
&\leq&\mathbb{E}l_0(Z,Y;W,b)1\left\{\forall h=1,\ldots, H,\; s.t.\;E(h)=0\right\}\\
&&+\mathbb{E}l_0(Z,Y;W,b)1\left\{\exists h=1,\ldots, H,\; s.t.\; E(h)=1\right\}1\{ \exists O(k\log k)\text{ choices of }i,\; s.t.\; X_i\neq 0 \}\\
&&+\mathbb{E}l_0(Z,Y;W,b)1\{ \exists \geq k\log k\text{ choices of }i,\; s.t.\; X_i\neq 0 \}1\{ |Y|< d \}\\
&&+\mathbb{E}l_0(Z,Y;W,b)1\{ |Y|> d \}.
\end{eqnarray*}
We bound the above terms one by one. Denote $b_{\max}=\max_h b_h$. Since we use the same upper bound for all hidden nodes to bound the probability that the hidden node gets unexpected zero/nonzero, we have
\begin{eqnarray*}
&&\mathbb{E}l_0(Z,Y;W,b)1\left\{\forall h=1,\ldots, c,\; s.t.\;E(h)=0\right\}\\
&&+\mathbb{E}l_0(Z,Y;W,b)1\left\{\exists h=1,\ldots, H,\; s.t.\; E(h)=1\right\}1\{ \exists O(k\log k)\text{ choices of }i,\; s.t.\; X_i\neq 0 \}\\\
&\leq&\mathbb{E}l_0(X,Y;\theta_0)+\underbrace{O(\sqrt{k/d})}_{\text{Noise in the active features}}\\
&&+O\left(\mathbb{E}(|X^{\top}(\theta_0-Ua)| + k\log k(Hm/d) b_{\max})^2 1\{ \exists O(k\log k)\text{ choices of }i,\; s.t.\; X_i\neq 0 \}\right)+o\\
&=&\mathbb{E}l_0(X,Y;\theta_0) +O(\sqrt{k/d})+O(\|\theta_0-Ua\|/\sqrt{d}) + O\left(  b_{\max}^2\left(\frac{k^2m^2 H}{d^2} \right)\left(\frac{Hmk\log k}{d}\right)^2 \right)+o,
\end{eqnarray*}
where $o$ is a negligible term and is caused by the noise $\xi$. We ignore this term in the following derivations.

Second,
\begin{eqnarray*}
&&\mathbb{E}l_0(Z,Y;W,b)1\{ \exists \geq k\log k\text{ choices of }i,\; s.t.\; X_i\neq 0 \}1\{ |Y|< d \}=o.
\end{eqnarray*}
And finally,
\begin{eqnarray*}
&&\mathbb{E}l_0(Z,Y;W,b)1\{ |Y|>  d \}\\
&=&O( \mathbb{E}(b_{\max} H+Y)^2 1\{ |Y|> d \} )\\
&=&O\left( \mathbb{E}_X\mathbb{E}_Y\left[(b_{\max} H+X^{\top}\theta_0+(Y-X^{\top}\theta_0))^2 1\{ |Y|>d \} \bigg| X=x\right]\right)\\
&=&O\left( \mathbb{E}(b_{\max} H+\|\theta_0-Ua\|+X^{\top}\theta_0)^2 1\{ |X^{\top}\theta_0|> d \}\right)\\
&=&O\left( \mathbb{E}(b_{\max} H+\|\theta_0\|_1)^2 1\{ |X^{\top}\theta_0|> d \}\right)\\
&=&O\left( (b_{\max} H+d)^2\frac{1}{\sqrt{k}(1+d^3)}\right),
\end{eqnarray*}
where the second line is because of the distribution of the noise $Y-X^{\top}\theta_0$, and the last line is based on Lemma \ref{lem:4}.

Since $b_{\max}=o(d/(H\sqrt{mk}))$ when $k\gg \log d$, we have
\begin{eqnarray*}
&&\mathbb{E}l_0(Z,Y;W,b)\\
&=& \mathbb{E}l_0(X,Y;\theta_0)+O(\|\theta_0-Ua\|/\sqrt{d}) +O(\sqrt{k/d})+ O\left(  b_{\max}^2\left(\frac{k^2m^2 H}{d^2} \right)\left(\frac{Hmk\log k}{d}\right)^2 \right)\\&&+O\left(\frac{(b_{\max} H+d)^2}{\sqrt{k}d^3}\right)+o\\
&=& \mathbb{E}l_0(X,Y;\theta_0)+O(\sqrt{k/d})+O(\|\theta_0-Ua\|/\sqrt{d}) + O\left(  b_{\max}^2\left(\frac{k^2m^2 H}{d^2} \right)\left(\frac{Hmk\log k}{d}\right)^2 \right)+o,
\end{eqnarray*}
and
\begin{eqnarray}
    \mathbb{E}l_0(Z,Y;W,b)=\mathbb{E}l_0(X,Y;\theta_0)+O(\|\theta_0-Ua\|/\sqrt{d})+O\left(\frac{Hm^3k^{3}\log^2 k}{d^2}\right)+O(\sqrt{k/d}),
\end{eqnarray}
from which we define $\psi$. 


For absolute loss and logistic regression, the error $\phi$ still holds.


\paragraph{Adversarial performance } There are on average $\Theta(k)$ active features in each data point, which is far less than the total $d$ features. As a result, each data point on average activates $\Theta( kHm/d )$ hidden nodes, and
$$
\mathbb{E}\|a^{\top}\text{diag}(\mathbb{I}(U^{\top}X,b))U^{\top}\|_2=\Theta(\mathbb{E}\sqrt{m}\|\theta_{\mathcal{X}}\|_2)=\Theta( \sqrt{mk} ).
$$
One the other hand, when we bound the error in clean model, we consider $|Y|<d$ and $|Y|>d$ cases. In the worse case, the increase of loss caused by the attack is $\epsilon\|\theta\|$, which is much smaller than $d$. As a result, whether or not we have the attack or not does not affect  $\psi$. 

As a result,
\begin{eqnarray*}
\mathbb{E}l_\epsilon(X,Y;W,b) = \mathbb{E}l_0(X,Y;\theta_0) + O(\psi)+ \Theta(\epsilon \sqrt{mk} ).
\end{eqnarray*}
On the other hand, taking $m=1$, one can also design a neural network such that
\begin{eqnarray}\label{eqn:super_upper}
\inf_{W',b'\in\mathcal{M}}\mathbb{E}l_\epsilon(X,Y;W',b')\leq \mathbb{E}l_\epsilon(X,Y;W,b)=\mathbb{E}l_0(X,Y;\theta_0)+O(\epsilon \sqrt{k})+O(\psi),
\end{eqnarray}
which finally indicates that
\begin{eqnarray*}
\mathbb{E}l_\epsilon(X,Y;W,b) = \inf_{W',b'\in\mathcal{M}}\mathbb{E}l_\epsilon(X,Y;W',b')+ O(\psi) + \Theta( \epsilon\sqrt{mk} ).
\end{eqnarray*}
\end{proof}

\begin{proof}[Proof of Theorem \ref{thm:adv}]
When $(W,b)\in\mathcal{M}$, the minimal non-zero value of $|U_{i,j}|$ is in $\Theta(d/(Hm^*))$. Assume on average there are $m$ features in each hidden node, then on average, there are $\Theta(Hkm/d)$ nodes are activated. In addition to the activated features, there are $\Theta(Hkm(m-1)/d)$ elements of $U_{i,h}$ leaked to the attacker. 

When these additional elements are all from different features, the increase of the loss is the smallest, which means that
$$ \epsilon\left\|\sum_{\text{activated h}}U_h\right\|\geq \Theta\left(\epsilon\sqrt{ k+\left(\frac{d}{Hm}\right)^2\left(\frac{Hkm(m-1)}{d}\right) }\right)=\Theta\left(\epsilon\left(\sqrt{k} + \frac{1}{2}\frac{d(m-1)}{Hm}\right)\right)+o. $$
When $H=o(\epsilon d^{3/2})$, $\epsilon d/H\gg \psi$. As a result, when an solution $(W,b)$ has an adversarial loss $O(\psi)$-close to $\min_{W',b'\in\mathcal{M}}\mathbb{E}l_\epsilon(Z,Y; M',b' )$, it also purifies most features, i.e, $m=1+o(1)$.

In terms of the clean performance, the result holds as
\begin{eqnarray*}
    \mathbb{E}l_0(Z,Y; W, b)
    \leq\min_{W',b'\in\mathcal{M}}\mathbb{E}l_\epsilon(Z,Y; W', b')+\Theta(\epsilon\sqrt{k})+O(\psi)
    \leq \mathbb{E}l_0(X,Y;\theta_0)+\Theta(\epsilon\sqrt{k})+O(\psi)+o.
\end{eqnarray*}
\end{proof}

\subsection{Proof for Contrastive Learning}
\begin{proof}[Proof of Lemma \ref{lem:contra_basic}]
Given $g(z,z')=x_1^{\top}T^{\top}Tx_2$, the contrastive loss becomes
\begin{eqnarray*}
&&\mathbb{E}\log(1+\exp(-X^{\top}T^{\top}TX)) + \mathbb{E}\log(1+\exp(X^{\top}T^{\top}TX')) \\
&=&\underbrace{\mathbb{E}\log(1+\exp(-X^{\top}PDP^{\top}X))}_{:=V_1} + \underbrace{\mathbb{E}\log(1+\exp(X^{\top}PDP^{\top}X'))}_{:=V_2}.
\end{eqnarray*}
To prove Lemma \ref{lem:contra_basic}, the key is to show that, fixing $tr(D)$, both $V_1$ and $V_2$ are minimized when $D\propto I_d$.

For $V_1$, as $\log(1+\exp(-v))$ is a convex function w.r.t. $v$, to show that $D\propto I_d$, we would like to show that for any $v\geq 0$,
\begin{eqnarray}\label{eqn:lem1}
\mathbb{E}\left[ X^{\top}PDP^{\top}X \big| \|P^{\top}X\|^2=v \right]=tr(D)v/d.
\end{eqnarray}
Based on the distribution of $X$, i.e., the distribution of each coordinate is symmetric and identical, we have
\begin{eqnarray*}
\mathbb{E}\left[ X^{\top}PDP^{\top}X \big| \|P^{\top}X\|^2=v \right]&=&\mathbb{E}\left[ \sum_{i=1}^d X_i^2 (PDP^{\top})_{i,i} \bigg| \|P^{\top}X\|^2=v \right]\\
&=&\mathbb{E}\left[ X_1^2\sum_{i=1}^d  (PDP^{\top})_{i,i} \bigg| \|P^{\top}X\|^2=v \right]\\
&=&\mathbb{E}\left[ X_1^2 tr(D) \bigg| \|P^{\top}X\|^2=v \right]\\
&=& tr(D)v/d.
\end{eqnarray*}
Consequently,
\begin{eqnarray*}
\mathbb{E}\left[\log(1+\exp(-X^{\top}PDP^{\top}X))\bigg| \|X\|^2=v\right]&\geq& \log\left(1+\exp\left(\mathbb{E}\left[-X^{\top}PDP^{\top}X\big| \|X\|^2=v\right]\right)\right)\\
&=& \log\left(1+\exp\left(-tr(D)v/d\right)\right),
\end{eqnarray*}
the equation holds when $D\propto I_d$.

Similarly, for $V_2$, $\log(1+\exp(v))$ is a convex function w.r.t. $v$, and
\begin{eqnarray*}
\mathbb{E}\left[\log(1+\exp(X^{\top}PDP^{\top}X'))\bigg| X^{\top}X'=v\right]&\geq& \log\left(1+\exp\left(\mathbb{E}\left[X^{\top}PDP^{\top}X'\big| X^{\top}X'=v\right]\right)\right)\\
&=& \log\left(1+\exp\left(tr(D)v/d\right)\right).
\end{eqnarray*}
As a result, fixing $tr(D)$, both $V_1$ and $V_2$ are minimized when taking $D\propto I_d$.
\end{proof}

\begin{proof}[Proof of Theorem \ref{lem:robust_similar}, loss for similar pairs]

For $A=\tau W^+$, we have
$$A=\tau W^+=\tau W^{\top}(WW^{\top})^{-1}=\tau U^{\top} M^{\top}( MUU^{\top}M^{\top}  )^{-1}=\tau U^{\top} ( UU^{\top} )^{-1}M^{\top} ,$$
which indicates that with probability tending to 1, $\sigma(Z^{\top} W,b)A=\tau X^{\top}UU^{\top} ( UU^{\top} )^{-1}M^{\top} + o=\tau X^{\top}M^{\top} + o$, where the term $o$ is negligible and is caused by the noise $\xi$.

With probability tending to 1, we have
    \begin{eqnarray*}
&&g_{W,b}(z+\delta_2,z')\\
&=&
g_{W,b}(z,z')+\epsilon\sgn\left(\frac{\partial l}{\partial f}\right)\| U\text{diag}(\mathbb{I}(x^{\top}U+\xi^{\top} W,b))AA^{\top}\text{diag}(\mathbb{I}(x^{\top}U+\xi^{\top} W,b))(U^{\top}x+M^{\top}\xi)\|,
 \end{eqnarray*}
 where with probability tending to 1, $\text{diag}(\mathbb{I}(x^{\top}U+\xi^{\top} W,b))=\text{diag}(\mathbb{I}(x^{\top}U ,\textbf{0}))$, and
 \begin{eqnarray*}
     &&\| U\text{diag}(\mathbb{I}(x^{\top}U+\xi^{\top} W,b))AA^{\top}\text{diag}(\mathbb{I}(x^{\top}U+\xi^{\top} W,b))(U^{\top}x+W^{\top}\xi)\|\\
     &\leq&\| U\text{diag}(\mathbb{I}(x^{\top}U+\xi^{\top} W,b))AA^{\top}\text{diag}(\mathbb{I}(x^{\top}U+\xi^{\top} W,b))U^{\top}x\|\\
     &&+\| U\text{diag}(\mathbb{I}(x^{\top}U+\xi^{\top} W,b))AA^{\top}\text{diag}(\mathbb{I}(x^{\top}U+\xi^{\top} W,b))W^{\top}\xi\|\\
     &=&\| U\text{diag}(\mathbb{I}(x^{\top}U+\xi^{\top} W,b))AMx\|+\| U\text{diag}(\mathbb{I}(x^{\top}U+\xi^{\top} W,b))AA^{\top}\text{diag}(\mathbb{I}(x^{\top}U+\xi^{\top} W,b))W^{\top}\xi\|.
 \end{eqnarray*}

Now we look into $\| U\text{diag}(\mathbb{I}(x^{\top}U+\xi^{\top} W,b))AMx \|$.

\paragraph{Active features} Assume the first coordinate of $x$ is non-zero. Since with probability tending to 1, all hidden nodes involving $x_1$ are activated, we have $U_{1,:}\text{diag}(\mathbb{I}(x^{\top}U+\xi^{\top} W,b))=U_{1,:}$, and
\begin{eqnarray}\label{eqn:contra_active}
U_{1,:}\text{diag}(\mathbb{I}(x^{\top}U+\xi^{\top} W,b))A Mx = U_{1,:}AMx= \tau x_1 .
\end{eqnarray}
\paragraph{Inactive features} Assume the second coordinate of $x$ is zero. We have for any active feature $i\in\mathcal{X}$,
\begin{eqnarray}\label{eqn:contra_inactive}
&&|U_{2,:}\text{diag}(\mathbb{I}(x^{\top}U+\xi^{\top} W,b))AM_{:,i}x_i|\\
&=&|\tau U_{2,:}\text{diag}(\mathbb{I}(x^{\top}U+\xi^{\top} W,b))U^{\top}(UU^{\top})^{-1}M^{\top}M_{:,i}x_i|\nonumber\\
&=&|\tau U_{2,:}\text{diag}(\mathbb{I}(x^{\top}U+\xi^{\top} W,b))U^{\top}(UU^{\top})^{-1}_{:,i}x_i|\nonumber\\
&=&\Theta(\tau |x_i|\alpha).\nonumber
\end{eqnarray}
For $i\in\mathcal{X}^c$, the value of the $i$th element of $U_{2,:}\text{diag}(\mathbb{I}(x^{\top}U+\xi^{\top} W,b))$ does not matter because $x_i=0$, i.e.,
\begin{eqnarray}\label{eqn:contra_inactive_other}
U_{2,:}\text{diag}(\mathbb{I}(x^{\top}U+\xi^{\top} W,b))AM_{:,i}x_i\equiv 0.
\end{eqnarray}
As a result, we have
\begin{eqnarray*}
    |U_{2,:}\text{diag}(\mathbb{I}(x^{\top}U+\xi^{\top} W,b))Ax|= O(\tau \alpha ),
\end{eqnarray*}
and thus the adversarial loss does not affected by whether the neural network is purified or not.
\paragraph{Noise} We now look into $\| U\text{diag}(\mathbb{I}(x^{\top}U+\xi^{\top} W,b))AA^{\top}\text{diag}(\mathbb{I}(x^{\top}U+\xi^{\top} W,b))W^{\top}\xi\|$.

From the assumptions, we know that 
\begin{equation}
    U_{i,:}\text{diag}(\mathbb{I}(x^{\top}U+\xi^{\top} W,b))A_{:,j} =\begin{cases} \tau &i=j\in\mathcal{X}\\
0 & i\neq j,i\in\mathcal{X}\\
\Theta(\alpha)& i\neq j,i\in\mathcal{X}^c,j\in\mathcal{X}\\
\Theta(\alpha^2)& \text{otherwise}
\end{cases}.\label{eqn:case}
\end{equation}
Consequently, taking proper value of $\alpha$, we get
$$\mathbb{E}\| U\text{diag}(\mathbb{I}(x^{\top}U+\xi^{\top} W,b))AA^{\top}\text{diag}(\mathbb{I}(x^{\top}U+\xi^{\top} W,b))W^{\top}\xi\|^2=o(k),$$
which indicates that the attack is dominated by $\| U\text{diag}(\mathbb{I}(x^{\top}U+\xi^{\top} W,b))A A^{\top}Mx \|$.

Next, we examine the performance of $g_{W,b}(z,z')$. We know that
\begin{eqnarray*}
    g_{W,b}(z,z')&=& (x^{\top}U+\xi^{\top}W)\text{diag}(\mathbb{I}(x^{\top}U+\xi^{\top} W,b))AA^{\top}\text{diag}(\mathbb{I}(x^{\top}U+(\xi')^{\top} W,b))(U^{\top}x+W^{\top}\xi')\\
    &=& x^{\top}U\text{diag}(\mathbb{I}(x^{\top}U+\xi^{\top} W,b))AA^{\top}\text{diag}(\mathbb{I}(x^{\top}U+(\xi')^{\top} W,b))U^{\top}x\\
    &&+\xi^{\top}W\text{diag}(\mathbb{I}(x^{\top}U+\xi^{\top} W,b))AA^{\top}\text{diag}(\mathbb{I}(x^{\top}U+(\xi')^{\top} W,b))W^{\top}\xi'\\
    &&+\xi^{\top}W\text{diag}(\mathbb{I}(x^{\top}U+\xi^{\top} W,b))AA^{\top}\text{diag}(\mathbb{I}(x^{\top}U+(\xi')^{\top} W,b))U^{\top}x\\
    &&+x^{\top}\text{diag}(\mathbb{I}(x^{\top}U+\xi^{\top} W,b))AA^{\top}\text{diag}(\mathbb{I}(x^{\top}U+(\xi')^{\top} W,b))W^{\top}\xi'.
\end{eqnarray*}
Based on (\ref{eqn:case}), one can see that with probability tending to one over the randomness of $(x,\xi,\xi')$,
\begin{eqnarray*}
    g_{W,b}(z,z')=x^{\top}U\text{diag}(\mathbb{I}(x^{\top}U+\xi^{\top} W,b))AA^{\top}\text{diag}(\mathbb{I}(x^{\top}U+(\xi')^{\top} W,b))U^{\top}x+o.
\end{eqnarray*}

Finally, when $\text{diag}(\mathbb{I}(x^{\top}U+\xi^{\top} W,b))\neq \text{diag}(\mathbb{I}(x^{\top}U ,\textbf{0}))$, if the eigenvalues of $UU^{\top}$ are bounded and bounded away from zero, we have
\begin{eqnarray*}
    \mathbb{E}g_{W,b}(Z,Z')1\{\exists E(h)=1\}=O(\|UU^{\top}\|P\{\exists E(h)=1\})=O(\psi).
\end{eqnarray*}
Different from supervised learning, in contrastive learning, we only care about the nodes which are related to the non-zero features in $x'$, so we only need to consider $O(Hk/d)$ hidden nodes rather than all the $H$ nodes. As a result, the value of $\psi$ gets smaller.


\end{proof}




\begin{proof}[Proof of Theorem \ref{lem:robust_similar}, loss for dissimilar pairs]

Similar to Theorem \ref{lem:robust_similar}, loss for similar pairs, we have
 \begin{eqnarray*}
     &&\| U\text{diag}(\mathbb{I}(x^{\top}U+\xi^{\top} W,b))AA^{\top}\text{diag}(\mathbb{I}((x')^{\top}U+(\xi')^{\top} W,b))(U^{\top}x'+W^{\top}\xi')\|\\
     &\leq&\| U\text{diag}(\mathbb{I}(x^{\top}U+\xi^{\top} W,b))AA^{\top}\text{diag}(\mathbb{I}((x')^{\top}U+(\xi')^{\top} W,b))U^{\top}x'\|\\
     &&+\| U\text{diag}(\mathbb{I}(x^{\top}U+\xi^{\top} W,b))AA^{\top}\text{diag}(\mathbb{I}((x')^{\top}U+(\xi')^{\top} W,b))W^{\top}\xi'\|\\
     &=&\| U\text{diag}(\mathbb{I}(x^{\top}U+\xi^{\top} W,b))AMx'\|+\| U\text{diag}(\mathbb{I}(x^{\top}U+\xi^{\top} W,b))AA^{\top}\text{diag}(\mathbb{I}((x')^{\top}U+(\xi')^{\top} W,b))W^{\top}\xi'\|.
 \end{eqnarray*}

     We look into $\| U\text{diag}(\mathbb{I}(x^{\top}U+\xi^{\top} W,b))AMx' \|$.

\paragraph{Active features} Assume the first coordinate of $x$ is non-zero. Since with probability tending to 1, all hidden nodes involving $x_1$ are activated, we have $U_{1,:}\text{diag}(\mathbb{I}(x^{\top}U+\xi^{\top} W,b))=U_{1,:}$, and
\begin{eqnarray}
U_{1,:}\text{diag}(\mathbb{I}(x^{\top}U+\xi^{\top} W,b))A Mx' = U_{1,:}AMx'= \tau x_1' .\nonumber
\end{eqnarray}
Consequently, 
\begin{eqnarray*}
    \mathbb{E}_{X'}\|U_{\mathcal{X}}\text{diag}(\mathbb{I}(x^{\top}U+\xi^{\top} W,b))A MX'   \|/\tau
    &=& \mathbb{E}\sqrt{ \sum_{i\in\mathcal{X}}|X_i'|^2  }\\
    &=& \mathbb{P}(\text{Only one $X_i'$ for $i\in\mathcal{X}$ is nonzero})/\sqrt{k}\\
    &&+\sqrt{2}\mathbb{P}(\text{Only two $X_i'$ for $i\in\mathcal{X}$  are nonzero})/\sqrt{k}\\
    &&+\sqrt{3}\mathbb{P}(\text{Only three $X_i'$ for $i\in\mathcal{X}$ are nonzero})/\sqrt{k}\\
    &&+\ldots,
\end{eqnarray*}
As a result,
\begin{eqnarray*}
    \mathbb{E}\sqrt{ \sum_{i\in\mathcal{X}}|X_i'|^2  }&\geq&\mathbb{P}(\text{Only one $X_i'$ for $i\in\mathcal{X}$ is nonzero})/\sqrt{k}\\
    &\geq& \frac{1}{c_l\sqrt{k}}\frac{k^2}{d}\\
    &=&\Theta(k^{3/2}/d),
\end{eqnarray*}
where $c_l>0$ is some constant number. Meanwhile,
\begin{eqnarray*}
    \mathbb{E}\sqrt{ \sum_{i\in\mathcal{X}}|X_i'|^2  }    &=& \mathbb{P}(\text{Only one $X_i'$ for $i\in\mathcal{X}$ is nonzero})/\sqrt{k}\\
    &&+\sqrt{2}\mathbb{P}(\text{Only two $X_i'$ for $i\in\mathcal{X}$  are nonzero})/\sqrt{k}\\
    &&+\sqrt{3}\mathbb{P}(\text{Only three $X_i'$ for $i\in\mathcal{X}$ are nonzero})/\sqrt{k}\\
    &&+\ldots,\\
    &\leq&\mathbb{P}(\text{Only one $X_i'$ for $i\in\mathcal{X}$ is nonzero})/\sqrt{k}\\
    &&+2\mathbb{P}(\text{Only two $X_i'$ for $i\in\mathcal{X}$  are nonzero})/\sqrt{k}\\
    &&+3\mathbb{P}(\text{Only three $X_i'$ for $i\in\mathcal{X}$ are nonzero})/\sqrt{k}\\
    &&+\ldots,\\
    &\leq& \frac{1}{\sqrt{k}}\left( \frac{k^2}{d}+ 2\left(\frac{k^2}{d}\right)^2+3\left(\frac{k^2}{d}\right)^3+\ldots\right)\\
    &=&\frac{1}{\sqrt{k}}\frac{k^2/d}{1-k^2/d}+\frac{1}{\sqrt{k}}\frac{(k^2/d)^2}{1-k^2/d}+\ldots\\
    &=&\Theta(k^{3/2}/d).
\end{eqnarray*}
Since both the upper bound and lower bound are in $\Theta(k^{3/2}/d)$, we have $\mathbb{E}_{X'}\|U_{\mathcal{X}}\text{diag}(\mathbb{I}(x^{\top}U+\xi^{\top} W,b))A MX'   \|=\Theta(\tau k^{3/2}/d)$.

\paragraph{Inactive features} Assume the second coordinate of $x$ is zero. From the assumption, we have 
\begin{eqnarray*}
    (U\text{diag}(\mathbb{I}(x^{\top}U+\xi^{\top} W,b))U^{\top}(UU^{\top})^{-1})_{i,j}=\begin{cases}
        1 & i,j\in\mathcal{X}\\
        O(\alpha) &i\neq j,     i\in\mathcal{X}^c,j\in\mathcal{X}\\
        O(\alpha^2) &\text{otherwise}
    \end{cases}.
\end{eqnarray*}
As a result, given $x$, 
\begin{eqnarray*}
    U_{2,:}\text{diag}(\mathbb{I}(x^{\top}U+\xi^{\top} W,b))AMX' = O_p\left(\frac{k}{d}\alpha + \alpha^2  \right),
\end{eqnarray*}
and $ \|U_{\mathcal{X}^c}\text{diag}(\mathbb{I}(x^{\top}U+\xi^{\top} W,b))A MX'\| =O_p( \alpha^2\sqrt{d}+\alpha k/\sqrt{d} )$.

\paragraph{Noise} Taking a proper value of $\alpha$, we get
$\mathbb{E}\| U\text{diag}(\mathbb{I}(x^{\top}U+\xi^{\top} W,b))AA^{\top}\text{diag}(\mathbb{I}((x')^{\top}U+(\xi')^{\top} W,b))W^{\top}\xi'\|^2$ is negligible.

When $\text{diag}(\mathbb{I}(x^{\top}U+\xi^{\top} W,b))\neq \text{diag}(\mathbb{I}(x^{\top}U ,\textbf{0}))$, the bound follows that same as Theorem \ref{lem:robust_similar}.

\end{proof}



\subsection{Proof for Downstream Task}

\begin{proof}[Proof of Proposition \ref{thm:downstream} ]
The arguments for Lemma \ref{thm:adv_ideal} holds for any $\theta$, not limited to $\theta_0$. As a result, when applying the neural network in new tasks, feature purification still preserves the adversarial robustness.
\end{proof}

\end{document}